\documentclass[10pt]{article}

\usepackage[table]{xcolor}
\usepackage{graphicx}%
\usepackage{multirow}%
\usepackage{amsmath,amssymb,amsfonts}%
\usepackage{amsthm}%
\usepackage{mathrsfs}%
\usepackage[title]{appendix}%
\usepackage{xcolor}%
\usepackage{textcomp}%
\usepackage{manyfoot}%
\usepackage{booktabs}%
\usepackage{algorithm}%
\usepackage{algorithmicx}%
\usepackage{algpseudocode}%
\usepackage{listings}%
\usepackage{subfigure}
\usepackage{multirow}
\usepackage{rotating}

\definecolor{codepurple}{rgb}{0.58,0,0.82}
\definecolor{backcolour}{rgb}{0.95,0.95,0.92}

\lstdefinestyle{pytorchstyle}{
    backgroundcolor=\color{backcolour},
    commentstyle=\color{red},
    keywordstyle=\color{codepurple},
    stringstyle=\color{orange},
    basicstyle=\ttfamily\footnotesize,
    breakatwhitespace=false,
    breaklines=true,
    captionpos=b,
    keepspaces=true,
    numbers=left,
    numbersep=5pt,
    showspaces=false,
    showstringspaces=false,
    showtabs=false,
    tabsize=2,
    language=Python,
    morekeywords={Tensor,Module,nn,Sequential,Linear,ReLU}
}

\newtheorem{theorem}{Theorem}

\theoremstyle{thmstyletwo}%

\theoremstyle{thmstylethree}%
\newtheorem{definition}{Definition}%

\usepackage[backref]{hyperref}
\usepackage[margin=1in]{geometry}
\def\BibTeX{{\rm B\kern-.05em{\sc i\kern-.025em b}\kern-.08em
    T\kern-.1667em\lower.7ex\hbox{E}\kern-.125emX}}

\begin{document}

\title{Knowledge Distillation Under Ideal Joint Classifier Assumption}

\author{Huayu Li$^{1}$ \quad Xiwen Chen$^{2}$ \quad Gregory Ditzler$^{3}$  \quad Janet Roveda$^{1}$ \quad  Ao Li$^{1}$ \\
$^1$The University of Arizona \quad $^2$Clemson University \quad $^3$ EpiSci\\
\texttt{\{hl459,meilingw,aoli1\}@arizona.edu}\\
\texttt{xiwenc@g.clemson.edu}\\
\texttt{gregory.ditzler@gmail.com}
}
\date{}

\maketitle

\begin{abstract}
Knowledge distillation constitutes a potent methodology for condensing substantial neural networks into more compact and efficient counterparts. Within this context, softmax regression representation learning serves as a widely embraced approach, leveraging a pre-established teacher network to guide the learning process of a diminutive student network. Notably, despite the extensive inquiry into the efficacy of softmax regression representation learning, the intricate underpinnings governing the knowledge transfer mechanism remain inadequately elucidated. This study introduces the 'Ideal Joint Classifier Knowledge Distillation' (IJCKD) framework, an overarching paradigm that not only furnishes a lucid and exhaustive comprehension of prevailing knowledge distillation techniques but also establishes a theoretical underpinning for prospective investigations. Employing mathematical methodologies derived from domain adaptation theory, this investigation conducts a comprehensive examination of the error boundary of the student network contingent upon the teacher network. Consequently, our framework facilitates efficient knowledge transference between teacher and student networks, thereby accommodating a diverse spectrum of applications.
\end{abstract}

\section{Introduction}

The advancements in deep neural networks in the field of computer vision are undeniable. Deep neural networks outperform classical machine learning algorithms in various tasks such as image recognition~\cite{krizhevsky2017imagenet,he2016deep}, semantic segmentation~\cite{long2015fully,chen2017deeplab, zhao2017pyramid}, and object detection~\cite{girshick2014rich, redmon2016you}. Unfortunately, the network's performance is usually determined by the number of parameters it has. Although more parameters result in a higher network capacity, they also lead to an increase in computational complexity and storage costs, which limits the use of deep neural networks on resource-limited hardware. As a result, reducing the size of a deep neural network while retaining the high performance of the large network is a critical to many applications.

Lightweight models, such as those discussed in~\cite{howard2017mobilenets,zhang2018shufflenet}, aim to reduce computational complexity through efficient network architecture design. Model compression techniques have been developed to address the complexity issue beyond architecture, including parameter pruning~\cite{han2015deep}, low-rank factorization \cite{tai2015convolutional}, and knowledge distillation \cite{bucilu2006model,hinton2015distilling}. The goal of knowledge distillation through model compression is to transfer the ``dark'' knowledge from a large, highly-accurate pre-trained teacher network to a small, high-speed student network with lower accuracy. 
Under knowledge distillation, the student's network is trained using the teacher's soft labels for guidance \cite{hinton2015distilling}, or using ``hints''  from the teacher's hidden layers \cite{romero2014fitnets} to produce  higher performance from the student than training the student's network with hard target labels alone.

An intuitive approach to improve knowledge distillation is to have the student model mimic the teacher model as closely as possible. Recent works in knowledge distillation can be divided into two main categories: representation and logits distillation. Representation-based techniques aim to extract richer information from the intermediate layers of the teacher model \cite{zagoruyko2016paying, chen2021cross, chen2021distilling}, or to better align the teacher and student features \cite{tung2019similarity, tian2019contrastive}. On the other hand, logits-based distillation techniques \cite{zhou2021rethinking, zhao2022decoupled} focus more on the statistical significance of the output logits and probabilities between the teacher and student networks.

Softmax Regression Representation Learning (SRRL) \cite{yang2021knowledge} is a type of logits-based distillation that matches the teacher and student logits, while using the teacher's classifier applied to the student's feature (penultimate layer output) and teacher's feature at the same time. SRRL offers a new perspective to improve the student's performance by not only using the teacher's logits, but also teacher's classifier.
A related method, Simple Knowledge Distillation (SimKD)  \cite{chen2022knowledge}, further improves SRRL by allowing the student to use a frozen, pre-trained teacher classifier for additional supervision. Despite the improvement achieved by SRRL-based methods, the mechanism that achieves this performance is not well understood because they are based on intuition rather than a solid theoretical explanation.

This paper explains the reasoning for using the teacher classifier as supervision in softmax regression-based distillation via a theory of domain adaptation. We summarized the existing methods and provides a theoretical analysis to understand the regularization effect of these methods. Our work presents an error bound that limits the student network's error by the teacher's error and two disagreement terms. 
As the error bounds are derived, the central motivation for SRRL and SimKD can be explained. We also introduce the concept of the ``Ideal Joint Classifier assumption'' to tighten the upper bound and better understand how to translate the theory of the upper bound to algorithms that improve the task of knowledge distillation. Using this bound, we presents a unified framework called Ideal Joint Classifier Knowledge Distillation (IJCKD), to connect SRRL and SimKD. The error bound is derived using a proof scheme from a theory of domain adaptation~\cite{ben2010theory}.

In conclusion, the authors' main contributions are:
\begin{itemize}
    \item Introducing a theoretical examination of techniques rooted in softmax regression-based methods, providing insights into their operation within the context of knowledge distillation. 
    \item Establishing an error bound that establishes a connection between teacher and student errors in the softmax regression setting, thereby offering a holistic approach to existing methods.
    \item Presenting the IJCKD framework, which unifies various softmax regression-based methods and opens avenues for future research extensions. 
\end{itemize}
\section{Related works}

The history of compressing machine models via training smaller models under the supervision of large models can be traced back to \cite{bucilu2006model}. In \cite{bucilu2006model}, MUNGE presented a method train a compact neural network to mimic the large and complex ensembles of classifiers. Taking advantage of the universal approximation property of neural networks \cite{hornik1989multilayer,scarselli1998universal}, MUNGE ``compresses'' the ensemble into smaller size networks. The concept of knowledge distillation was formally defined in \cite{hinton2015distilling}, which defined the softmax output of the teacher network as the ``knowledge''. The student is supervised by hard labels from the oracle, and the soft labels from teacher outputs to achieve higher accuracy than a network trained without distillation. The knowledge distillation objective can be formulated as: $l^{all} = l^{ce}+\lambda\cdot l^{kd}$, where $l^{ce}$ is the cross entropy between student output probabilities and the hard label, and the $l^{kd}$ is the Kullback-Leibler (KL) divergence between the student and teacher output probabilities weighted by a factor $\lambda$. Typically, the $l_{kd}$ term has a temperature factor $\tau$ that softens the output probabilities. The following distillation works focus on better use of the teacher's logits. The work proposed in \cite{kim2021comparing} comprehensively compared the difference and association between KL divergence between the softened output probabilities and mean square error (MSE) between the output logits prior to the softmax activation. Based upon their findings, Kim et al. proposed replacing the original KL divergence term with the MSE between the student and teacher logits \cite{kim2021comparing}. Their experimental results showed that the logits loss provided better distillation and performance. Further, decoupled knowledge distillation (DKD)~\cite{zhao2022decoupled} divides the KL divergence into the target and non-target loss according to the ground truth label. Weighted soft label distillation (WSLD)~\cite{zhou2021rethinking} analyzed the regularization effect of knowledge distillation through a bias-variance decomposition and proposed to rescale the knowledge distillation loss with respect to the regularization samples. \textcolor{black}{Guo et al. \cite{gou2023multi} proposed MTKD-SSR to use multi-stage learning and self-reflection in knowledge distillation. Their results showed significant improvements over existing methods in various computer vision tasks.}

Feature matching (e.g., feature based distillation or representation distillation) aims to have the student network approximate intermediate features of the teacher's network. FitNet~\cite{romero2014fitnets} directly uses the teacher's features as hints to supervise the student's feature representations. A generic feature distillation optimization task can be defined as: $l^{all} = l^{ce}+\lambda\cdot l^{fm}$, where $l^{fm}$ is typically the distance between the intermediate features from teachers and students. Several representative feature based distillation methodologies were proposed after FitNet. 
Overhaul of Feature Distillation (OFD) \cite{heo2019comprehensive} investigated the effect of the distillation feature position and distance function. Relational Knowledge Distillation (RKD) \cite{park2019relational} designed a new loss from distance- and angle-wise perspectives to extract relations between the feature representations. Contrastive Representation Distillation (CRD) \cite{tian2019contrastive} deployed contrastive learning to maximize the mutual information between the student and teacher representations. 

The focus of this paper is on a specific type of knowledge distillation, which we refer to as Softmax Regression-based Distillation. SRRL \cite{yang2021knowledge} uses the teacher's classifier as additional supervision to regularize the student's features. This approach ensures that the student features produce the same outputs as the teacher features when scored by the teacher classifier. By being supervised by the teacher classifier, SRRL achieves outstanding performance with simply two MSE loss terms between the student and teacher penultimate layer features, and the logits produced by the student and teacher features when passed through the teacher classifier. Following this, SimKD \cite{chen2022knowledge} adopts a more aggressive approach known as the reused teacher classifier. SimKD abandons the student classifier and lets the student use the frozen pre-trained teacher classifier as its own classifier. 
The student's network is trained with the feature matching loss to learn the same features as the teacher, and the teacher's network is frozen. SimKD also proposes a multi-layer convolutional module with little computational overhead to better align the teacher and student features.

\section{Knowledge Distillation Under Ideal Joint Classifier Assumption}

Given two deep neural networks $f_t$ and $f_s$ as the teacher and student networks, respectively. The knowledge distillation task aims to train a student network, $f_s$, from the output of a significantly larger teacher network $f_t$. More generally, we wish the student network to mimic the teacher. Therefore, upper-bounding the student network's error with the teacher's error is essential to study the learning framework of knowledge distillation. 
Let the dataset $\mathcal{D}=\left \{\mathcal{X}, \mathcal{Y}\right \}$ contain $n$ data pairs of inputs $\mathcal{X}=\left \{x_i\right \}^{n}_{i=1}$ and the ground truth labels $\mathcal{Y}=\left \{y_i\right \}^{n}_{i=1}$. A deep neural network $f$ typically consists of a feature extractor $\Phi:\mathcal{X}\rightarrow \mathcal{Z}$ to map the inputs $x$ to the latent representations $z$ and a linear layer (a classifier) $g: \mathcal{Z}\rightarrow \mathcal{O}$ to classify the latent representations $z$ into their output logits $o$. We also denote $p$ as the output probability of the classifier after softmax activation, where $p_k=\frac{exp(o_k)}{\sum_j exp(o_j)}$. 

The error (also called risk) of a neural network $f$ is measured by the disagreement between its predictions and the ground truth labels. Let us define an error $\epsilon(f)$:  
\begin{align}
    \epsilon(f) &= \mathbb{E}_{\mathcal{D}}[|f(x)-y|] = \mathbb{E}_{\mathcal{D}}[|g\circ\Phi(x)-y|] =\epsilon(g\circ\Phi).
    \label{eq:abs loss}
\end{align}
To guarantee student performance under the supervision of both the ground truth labels and the teacher, we seek to limit student error in terms of teacher error. For a student network $f_s$ and teacher $f_t$, we have a basic assumption that the teacher network has a larger capacity than the student, e.g., the teacher has a more significant number of parameters which let the teacher performs better than the student on dataset $\mathcal{D}$. Thus, we could bound the student error in terms of the teacher error. 

\begin{theorem}
\label{theorem_1} 
For a teacher network $f_t=g_t\circ\Phi_t$, a student network $f_s=g_s\circ\Phi_s$, the student error can be bounded as:
\begin{align}
    \epsilon(f_s) &\leq \epsilon(f_t)+\underbrace{\mathbb{E}_{\mathcal{D}}[|g_s\circ\Phi_s(x)-g_t\circ\Phi_s(x)|]}_{\Delta_1} + \underbrace{\mathbb{E}_{\mathcal{D}}[|g_t\circ\Phi_s(x)-g_t\circ\Phi_t(x)|]}_{\Delta_2}
\end{align}
\end{theorem}

\begin{proof}
\begin{align*}
    \epsilon(f_s) &= \epsilon(f_s)+(\epsilon(f_t)-\epsilon(f_t))+(\epsilon(g_t\circ\Phi_s)-\epsilon(g_t\circ\Phi_s))\\
    &\leq \epsilon(f_t)+|\epsilon(g_s\circ\Phi_s)-\epsilon(g_t\circ\Phi_s)|+|\epsilon(g_t\circ\Phi_s)-\epsilon(g_t\circ\Phi_t)| \\
    & \leq \epsilon(f_t)+\underbrace{\mathbb{E}_{\mathcal{D}}[|g_s\circ\Phi_s(x)-g_t\circ\Phi_s(x)|]}_{\Delta_1} + \underbrace{\mathbb{E}_{\mathcal{D}}[|g_t\circ\Phi_s(x)-g_t\circ\Phi_t(x)|]}_{\Delta_2}
\end{align*}
\end{proof}
This error bound consists of three distinct components: 
(a) the first term represents the teacher's risk on the dataset $\mathcal{D}$, 
(b) the second term, $\Delta_1$, captures the difference between the teacher and student classifiers, and 
(c) the third term, $\Delta_2$, measures the discrepancy between the teacher and student models when scoring the same inputs. This bound reveals the requirement for training a good student network under knowledge distillation frameworks. First, the $\Delta_1$ term tells us the learned student's classifier should to be similar to the teacher's classifier, which means that the student's and teacher's classifiers are supposed to have close outputs with respect to the same input features. The $\Delta_2$ term represents that a suitable student feature is expected to lead to similar predictions under the teacher's classifier as the teacher representation. Further, this bound limits the error, but more importantly, it shows us three individual terms that we can seek to minimize to improve the task of knowledge distillation. 

In practice, a pairwise loss function $l$ is used for capturing the disagreements, including the MSE loss and the softmax cross-entropy loss, which are defined as:
\begin{align*}
    &l^{\text{mse}}(y_i,f(x_i))= \left \|y_i - f(x_i) \right \|_2^2,\hspace{1em}
    l^{\text{ce}}(y_i,f(x_i)) = -\text{log}\,p(Y=y_i|x_i)
\end{align*}
With a minor deviation from the original notation, the RHS of the inequality of Theorem \ref{theorem_1} can be (approximately) rewritten as:
\begin{align}
\label{eq:modified rhs}
    \epsilon(f_t)+\underbrace{\mathbb{E}_{\mathcal{D}}\left [ l(g_s\circ\Phi_s(x),g_t\circ\Phi_s(x)) \right ]}_{\Delta_1}+\underbrace{\mathbb{E}_{\mathcal{D}}\left [ l(g_t\circ\Phi_s(x),g_t\circ\Phi_t(x)) \right ]}_{\Delta_2}.
\end{align}
Note that we are not claiming a strict inequality here as we did in Theorem \ref{theorem_1}, and we have substituted the absolute loss in Eq. \eqref{eq:abs loss} with a generic loss $l(\cdot)$. 
At this point, we can make two observations regarding how SRRL and SimKD address the error bound. SRRL addresses the $\Delta_2$ term by using a softmax regression loss, $l^{sr}(g_t\circ\Phi_s(x),g_t\circ\Phi_t(x))$. On the other hand, SimKD mitigates the $\Delta_1$ term by allowing the student to share the teacher's pre-trained classifier, then minimizes the feature matching loss $l^{fm}(\Phi_s(x), \Phi_t(x))$. Based on these observations, we first address the $\Delta_1$ term as our core assumption. The disparity between the teacher and student classifiers should be small. Hence, we can assume that there exists an ideal joint classifier that achieves the lowest risk for both the student and teacher representations.

\begin{definition}
The ideal joint classifier of the student and teacher representations on a dataset $\mathcal{D}$ is
\begin{align}
\label{eq:ideal joint classifier}
    \hat{g} = \arg\min_{g} \mathbb{E}_{\mathcal{D}} [l(y,g\circ\Phi_t(x))+l(y,g\circ\Phi_s(x))]
\end{align}
\end{definition}
We assume that there exists an ideal joint classifier that could reach a low risk with both the teacher's and student's representations. The ideal joint classifier assumption provides insight into analyzing the student's performance. The student will be unlikely to learn effectively under the teacher's guidance when there is a large discrepancy between the student's and the teacher's classifiers. Under the ideal joint classifier assumption, the student's risk can be bounded by the teacher's risk and the discrepancy between the outputs of the ideal joint classifier over the student and teacher representations. 
SimKD provides a natural choice for reusing the teacher's classifier as the ideal joint classifier, which we use in our approach. Then, the bound is re-expressed with the ideal joint classifier. This highlights the significance of the ideal joint classifier assumption in understanding the student's performance and serves as a foundation for further analysis. 

Now that we can substitute the ideal joint classifier into the RHS of the modified inequality of Theorem \ref{theorem_1} described in Eq. \eqref{eq:modified rhs} and demonstrate how it allows us to bound the student's error. Given a teacher and a student feature extractor $\Phi_s$ and $\Phi_t$, respectively, and an ideal joint classifier $\hat{g}$, the RHS of the inequality of Theorem \ref{theorem_1} can be further rewritten as:
\begin{align}
\label{eq: rhs with ij}
    \epsilon(\hat{g}\circ\Phi_t)+\mathbb{E}_{\mathcal{D}}\left [ l(\hat{g}\circ\Phi_s(x),\hat{g}\circ\Phi_t(x)) \right ].
\end{align}
We can look at the SRRL learning objective. Recall that SRRL defines its learning objective as follows:
\begin{align}
    f_s = g_s\circ\Phi_s = \arg\min_{g,\Phi} \mathbb{E}_{\mathcal{D}}[l^{ce}(y,g\circ\Phi(x))+\alpha l^{sr}(g_t\circ\Phi_t(x),g_t\circ\Phi(x))]. 
\end{align}
The softmax regression loss, denoted as $l^{sr}$, is the key component of the SRRL's objective, and $\alpha$ controls the trade-off between the losses. 
We rewrite the SRRL objective function with the ideal joint classifier setting the softmax regression loss to cross-entropy loss:
\begin{align}
    \Phi_s= \arg\min_{\Phi} \mathbb{E}_{\mathcal{D}}[ &l^{ce}(y,\hat{g}\circ\Phi(x)) + \alpha l^{ce}(\hat{g}\circ\Phi_t(x),\hat{g}\circ\Phi(x)) ], 
\end{align}
The objective function can be seen as training the student's feature extractor $\Phi_t$ with the smoothed labels based on the teacher's prior over each class. If we followed the original implementations of SRRL~\cite{yang2021knowledge}, and set the $l^{sr}$ be mean square error loss $l^{mse}$, the objective function is defined as:
\begin{align}
\label{eq:IJCKD}
    \Phi_s = \arg\min_{\Phi} \mathbb{E}_{\mathcal{D}}[&l^{ce}(y,\hat{g}\circ\Phi(x))+\alpha l^{mse}(o^s,o^t)], 
\end{align}
where $o^s$ and $o^t$ denote the logits produced by the teacher and student models, respectively. 
Recent work has observed that using MSE loss between the logits is more effective than cross-entropy loss \cite{yang2021knowledge, kim2021comparing}. Then SRRL, which is closely related to SimKD through the incorporation of the ideal joint classifier assumption, can be formalized into the IJCKD framework. However, there is a key difference in the learning objective of IJCKD compared to SimKD, as Eq. \eqref{eq: rhs with ij} and the ideal joint classifier assumption suggest the minimization of both cross-entropy loss between the student outputs and the hard label, as well as the feature/logits matching loss simultaneously. \textcolor{black}{
Algorithms \ref{alg:srrl}, \ref{alg:simkd}, and \ref{alg:IJCKD} show the implementations of SRRL, SimKD, and IJCKD, respectively. 
Further, Figure \ref{fig:IJCKD} provides an intuitive illustration of distinctions between SRRL, SimKD, and IJCKD.}

\begin{algorithm}
\begin{lstlisting}[style=pytorchstyle]
#input: x, hard label: y
#student network: net_s, teacher network: net_t
#connector for align teacher and student channels
logits_s, feat_s = net_s(x)
feat_s = connector(feat_s)
logits_t, feat_t = net_t(x)
logits_aux = net_t.fc(avg_pool(feat_s))

loss_fm = F.mse_loss(feat_s, feat_t)
loss_lm = F.mse_loss(logits_aux, logits_t)
loss_ce = F.cross_entropy(logits_s,y)

#alpha, beta are used to scale the losses

loss_srrl = loss_ce + alpha*loss_lm + beta*loss_fm

\end{lstlisting}
\caption{PyTorch Code for SRRL}
\label{alg:srrl}
\end{algorithm}

\begin{algorithm}
\begin{lstlisting}[style=pytorchstyle]
#input: x
#student network: net_s, teacher network: net_t
#connector for align teacher and student channels
_, feat_s = net_s(x)
feat_s = connector(feat_s)
_, feat_t = net_t(x)

loss_fm = F.mse_loss(feat_s, feat_t)

loss_simkd = loss_fm
\end{lstlisting}
\caption{PyTorch Code for SimKD}
\label{alg:simkd}
\end{algorithm}
\begin{algorithm}
\begin{lstlisting}[style=pytorchstyle]
#input: x, hard label: y
#student network: net_s, teacher network: net_t
#connector for align teacher and student channels
_, feat_s = net_s(x)
feat_s = connector(feat_s)
logits_t, feat_t = net_t(x)
logits_s = net_t.fc(avg_pool(feat_s))

loss_lm = F.mse_loss(logits_s, logits_t)
loss_ce = F.cross_entropy(logits_s,y)

#alpha is used to scale the losses

loss_ijckd = loss_ce + alpha*loss_lm

\end{lstlisting}
\caption{PyTorch Code for IJCKD}
\label{alg:IJCKD}
\end{algorithm}

\begin{figure}[!htbp]
    \centering
    \includegraphics[width=.95\textwidth]{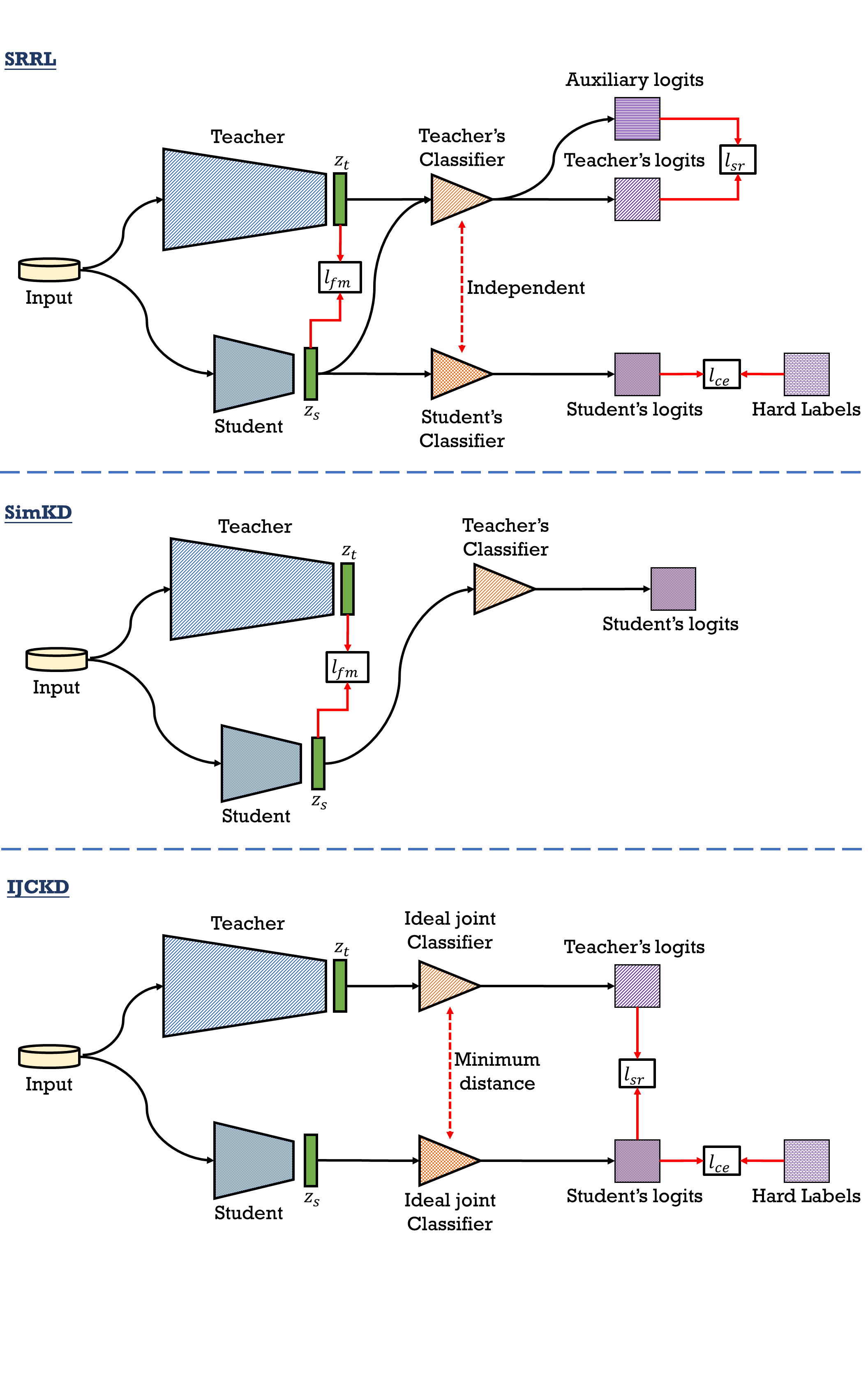}
    \caption{\textcolor{black}{An illustration to show the differences between the SRRL, SimKD, and IJCKD frameworks.}}
    \label{fig:IJCKD}
\end{figure}

\section{Experiments}

This section presents experiments that demonstrate the effectiveness of our proposed IJCKD approach. We first present the results on standard benchmark datasets and compare them with several representative state-of-the-art approaches (CIFAR-100 \cite{krizhevsky2009learning}, and ImageNet \cite{deng2009imagenet}). We compare IJCKD with SimKD in a separate subsection since the connector's design is a key factor in SimKD's performance. Additionally, we conduct experiments where a student network with only a pre-trained teacher classifier to verify whether the pre-trained teacher classifier can be the optimal classifier for the student network. We implemented the IJCKD in Python and the Pytorch framework \cite{paszke2019pytorch}. The experiments were conducted on a workstation with an Intel i9-12900 processor, and two Nvidia RTX 3090 graphics cards with 24GB RAM.

\subsection{Results on CIFAR-100}
We followed the standard training procedure adopted by previous works \cite{zhou2021rethinking, tian2019contrastive, zhao2022decoupled}. Specifically, we train the networks for 240 epochs, and the learning rate is decayed multiplying 0.1 at the 150th, 180th, and 210th epochs, respectively. The initial learning rate is set to 0.01 when ShuffleNet is the backbone \cite{zhang2018shufflenet, ma2018shufflenet}, and an initial learning rate to 0.05 for all other backbones. We used SGD with Nesterov's momentum of 0.9, a mini-batch size of 64, and a weight decay of $5\times10^{-4}$. 
To align the teacher and student feature channels, we applied a 1$\times$1 convolutional layer followed by batch normalization \cite{ioffe2015batch}, and ReLU activations \cite{nair2010rectified} to the student's features. We report the top-1 classification accuracy in Table \ref{tabel:cifar-100}, and each result is reported over five run average. Our study primarily focuses on IJCKD and SRRL, with additional evaluations of several other recent state-of-the-art methods including FitNet \cite{romero2014fitnets}, KD \cite{hinton2015distilling}, VID \cite{ahn2019variational}, RKD \cite{park2019relational}, PKT \cite{passalis2018learning}, OFD \cite{heo2019comprehensive}, CRD \cite{tian2019contrastive}, WSLL \cite{zhou2021rethinking}, and DKD \cite{zhao2022decoupled}. We also present the reproduced results of SimKD with 1$\times$1 connector. We compare distillation approaches across models with both the same and different architectures. To ensure a fair comparison and showcase the versatility of IJCKD, we set $\alpha=1$ without an exhaustive search as the scaling factor for the logits matching loss in all teacher-student pairs.

Our experimental results offer valuable insights into the effectiveness of IJCKD compared to the SRRL method. Focusing on the differences between these two approaches, it becomes evident that IJCKD consistently outperforms SRRL across various teacher-student pairs evaluated in our study, as detailed in Table \ref{tabel:cifar-100}. For instance, let's consider the ResNet32x4$\rightarrow$ResNet8x4 teacher-student pair, a particularly challenging scenario. Here, IJCKD achieves an impressive top-1 accuracy of 76.52\%, surpassing SRRL by a substantial margin of 0.6\%. These results underscore the remarkable efficacy of IJCKD in knowledge transfer and student network improvement. A key element contributing to IJCKD's superior performance lies in the combination of logits matching and cross-entropy losses under the Ideal Joint Classifier Assumption. By aligning the output logits of the student with those of the teacher, IJCKD enhances the student's ability to capture intricate details and fine-grained information from the teacher model. This effectively reduces the information gap between the teacher and student, resulting in more accurate predictions and higher classification accuracy. It's noteworthy that these improvements are not limited to specific teacher-student pairs; rather, IJCKD consistently demonstrates its superiority across the range of teacher architectures tested in this work. 

This consistency highlights the robustness and general applicability of IJCKD as a knowledge distillation technique. Furthermore, when comparing IJCKD to other knowledge distillation methods in Table \ref{tabel:cifar-100}, we can see that IJCKD often achieves top-1 accuracy that surpasses or closely rivals the best-performing alternatives. This suggests that IJCKD holds great promise for enhancing the performance of student networks across a wide array of architectures and domains. In summary, our experimental findings strongly support the notion that IJCKD, by leveraging the Ideal Joint Classifier Assumption, consistently outperforms SRRL and other knowledge distillation techniques in improving student network accuracy.

\begin{sidewaystable}[!htbp]
\caption{Top-1 accuracy (\%) on CIFAR-100.}
\label{tabel:cifar-100}
\centering
\begin{tabular}{c|ccccc|ccc}
\hline
        & \multicolumn{5}{c|}{Same architecture style}              & \multicolumn{3}{c}{Different architecture style} \\ \hline
Teacher & WRN-40-2 & ResNet56 & ResNet110 & ResNet110 & ResNet32x4 & ResNet32x4     & ResNet32x4     & WRN-40-2        \\
        & 75.61    & 72.34    & 74.31     & 74.31     & 79.42      & 79.42          & 79.42          & 75.61           \\ \hline
Student & WRN-40-1 & ResNet20 & ResNet20  & ResNet32  & ResNet8x4  & ShuffleNetV1   & ShuffleNetV2   & ShuffleNetV1    \\
        & 71.98    & 69.06    & 69.06     & 71.14     & 72.50      & 70.50          & 71.82          & 70.50           \\ \hline
FitNet  & 72.24    & 69.21    & 68.99     & 71.06     & 73.50      & 73.59          & 73.54          & 73.73           \\
KD      & 73.54    & 70.66    & 70.67     & 73.08     & 73.33      & 74.07          & 74.45          & 74.83           \\
VID     & 73.30    & 70.38    & 70.16     & 72.61     & 73.09      & 73.38          & 73.40          & 73.61           \\
RKD     & 72.22    & 69.61    & 69.25     & 71.82     & 71.9       & 72.28          & 73.21          & 72.21           \\
PKT     & 73.45    & 70.34    & 70.25     & 72.61     & 73.64      & 74.10          & 74.69          & 73.89           \\
OFD     & 72.38    & 69.47    & 69.53     & 70.98     & 73.17      & 73.55          & 74.31          & 73.34           \\
CRD     & 74.14    & 71.16    & 71.46     & 73.48     & 75.51      & 75.11          & 75.65          & 76.05           \\
WSLL    & 74.48    & 72.15    & 72.19     & 74.12     & 76.05      & 75.46          & 75.93          & 76.21           \\
DKD     & 74.81    & 71.97    & 72.31     & 74.11     & 76.32      & 76.45          & 77.07          & 76.70           \\ \hline\hline
SRRL    & 74.75    & 71.44    & 71.51     & 73.80     & 75.92      & 75.66          & 76.40          & 76.61           \\
SimKD   & 71.76    & 68.86    & 69.98     & 72.85     & 75.91      & 76.48          & 76.55          & 75.65           \\
IJCKD   & 75.14    & 71.73    & 71.76     & 73.98     & 76.52      & 76.51          & 76.56          & 76.84           \\ \hline
\end{tabular}
\end{sidewaystable}

\subsection{Compare with SimKD}
In this section, we conduct an in-depth analysis comparing the performance of IJCKD and SimKD across various connector architectures. Our primary goal is to assess the adaptability and effectiveness of these methods concerning different connector designs. To this end, we consider three distinct connector architectures: 1$\times$1Conv, 1$\times$1Conv-1$\times$1Conv, and 1$\times$1Conv-3$\times$3Conv-1$\times$1Conv, where the notation denotes the presence and size of convolutional layers in the connector. The results presented in Table \ref{tabel:simkd} offer valuable insights into the comparative performance of IJCKD and SimKD with varying connector architectures. Notably, these findings highlight the versatility and adaptability of IJCKD, emphasizing its ability to excel across different connector designs.

IJCKD consistently outperforms SimKD, achieving superior performance even with a simpler 1$\times$1Conv connector. This result underscores the robustness of the IJCKD approach, as it demonstrates its effectiveness in distillation across various architectures, even when minimal convolutional layers are involved. Conversely, SimKD struggles to match the adaptability of IJCKD. Specifically, SimKD fails to achieve competitive results when confronted with 1$\times$1Conv and 1$\times$1Conv-1$\times$1Conv connectors. This limitation is indicative of SimKD's primary approach, which focuses on training the student to mimic the teacher's representation through a regression-like process. When paired with connectors that differ significantly from the teacher's architecture, SimKD faces challenges in achieving satisfactory performance. In contrast, IJCKD showcases itself as a robust and reliable knowledge distillation method capable of effectively handling varying connector architectures. The improved performance consistently observed across different connector designs reaffirms IJCKD's adaptability and highlights its potential as a versatile tool for knowledge transfer in diverse scenarios.

In summary, our comprehensive evaluation of IJCKD and SimKD across different connector architectures underscores IJCKD's superior adaptability and performance. These findings indicate that IJCKD is well-suited for scenarios involving a wide range of connector designs, making it a valuable choice for knowledge distillation tasks with varying model architectures and complexities.

\begin{table*}[!htbp]
\caption{Comparison with SimKD with different connector architecture.}
\label{tabel:simkd}
\centering
\begin{tabular}{c|cc|cc}
\hline
Teacher/ Student                             & \multicolumn{2}{c|}{\begin{tabular}[c]{@{}c@{}}WRN-40-2\\ WRN-40-1\end{tabular}}  &\multicolumn{2}{c}{\begin{tabular}[c]{@{}c@{}}resnet32x4\\resnet8x4\end{tabular}}\\\hline
Methods           & IJCKD             & SimKD             & IJCKD               & SimKD             \\\hline
1$\times$1Conv                               & 75.14                  &  71.76                 &76.52                     & 75.91              \\
1$\times$1Conv-1$\times$1Conv                & 75.33                  & 72.23                  &77.10                     & 76.21            \\
1$\times$1Conv-3$\times$3Conv-1$\times$1Conv & 75.57                  & 75.48                  &  77.76                   & 77.46           \\
\hline
\end{tabular}
\end{table*}

\subsection{Ablation Study}
\textcolor{black}{In this ablation study, the weight assigned to the cross-entropy loss and logits matching loss was assessed. Initially, we set the logits matching loss ($\alpha_{lm}$) weight to 1 and then we adjust the cross-entropy loss  ($\alpha_{ce}$) weight from 0 to 1. When $\alpha_{ce} = 0$, the student network is trained exclusively under the teacher's logits supervision, a scenario labeled as ``SR only.'' With increasing $\alpha_{ce}$, the influence of the hard ground truth label becomes more pronounced in the learning process. This approach aligns with the ideal joint classifier assumption, which posits that the student network should be trained under both teacher and ground truth supervision. Additionally, the study explored the scenario of learning without the teacher's logits, termed 'CE only', where $\alpha_{lm}$ is set to 0. The results, specifically the top-1 accuracy for the ResNet32x4-ResNet8x4 teacher-student pair on the CIFAR-100 dataset, are presented in Table \ref{tabel:ablation}. Figure \ref{fig:alpha} illustrates the corresponding training and validation top-1 accuracy curves. Notably, both 'SR only' and 'CE only' conditions were less effective than a linear combination of both losses, corroborating the proposed assumption and error bound. Optimal results were achieved with $\alpha_{ce} = 0.2$ and $\alpha_{ce} = 1.0$ for the best validation and training accuracy, respectively. However, it is crucial to acknowledge that the ideal balance of these losses varies depending on the teacher-student pair and dataset. Thus, future research should focus on hyperparameter tuning to identify the most effective loss scales for specific contexts, aiming to enhance distillation performance.}

Further, in addition to exploring the impact of $\alpha$, we delved into the influence of different combinations of feature and logits matching losses in our study. Our default choice, the naive softmax regression loss $l^{sr}$ (MSE loss between teacher and student logits), was compared against the feature matching loss $l^{fm}$ (MSE loss between teacher and student features). Additionally, we examined the combination of both losses. These experiments were conducted with the ResNet32x4-ResNet8x4 pair on CIFAR-100, and the resulting top-1 accuracy outcomes are presented in Table \ref{tabel:ablation}, while corresponding training and validation curves can be found in Figure \ref{fig:comb}. Our results illuminate the significant impact of the chosen matching loss on IJCKD's performance. Intriguingly, the combination of feature and logits matching losses consistently outperformed the individual losses, highlighting the benefits of jointly optimizing both loss components in practical scenarios. This finding underscores the importance of considering and carefully selecting the matching losses when employing IJCKD, as this choice can significantly affect the distillation process's overall effectiveness.

\textcolor{black}{Our study further explored various logits matching losses, with results presented in Table \ref{tabel:loss_func}. This comparison encompassed MSE, negative cosine similarity, and cross-entropy as logits matching losses. The scale for both MSE and cross-entropy was set to 1, whereas for negative cosine similarity, we assigned an $\alpha_{lm}$ value of 10. Our findings indicated that negative cosine similarity, with its adjusted scale factor, consistently surpassed MSE and cross-entropy in achieving higher top-1 accuracy for both teacher-student pairings. Notably, the cross-entropy loss underperformed compared to the other two methods, aligning with observations reported in the SRRL paper \cite{yang2021knowledge}. In summary, this ablation study underscores the significance of the appropriate combination of logits matching loss and cross-entropy loss. It also highlights the effectiveness of varying logits matching losses in different distillation contexts.}

\begin{table*}[!htbp]
\caption{Top-1 accuracy for different loss scale and combination.}
\label{tabel:ablation}
\centering
\begin{tabular}{c|cccccc}
\hline
\multirow{2}{*}{\begin{tabular}[c]{@{}c@{}}$\alpha_{ce}$\\ top-1\end{tabular}} & SR only              & $0.1$            & $0.2$            & $0.5$            & $1.0$               & CE only              \\
                                                                               & 75.91            & 76.36            & 76.61            & 76.09            & 76.52               & 74.46                \\\hline
\multirow{2}{*}{\begin{tabular}[c]{@{}c@{}}Loss comb\\ top-1\end{tabular}}     & \multicolumn{2}{c}{$l^{ce}+l^{sr}$} & \multicolumn{2}{c}{$l^{ce}+l^{fm}$} & \multicolumn{2}{c}{$l^{ce}+l^{sr}+l^{fm}$} \\
                                                                               & \multicolumn{2}{c}{76.52}           & \multicolumn{2}{c}{76.47}           & \multicolumn{2}{c}{76.86}   \\\hline              
\end{tabular}
\end{table*}

\begin{table}[]
\caption{Top-1 accuracy for different logits matching loss.}
\label{tabel:loss_func}
\centering
\begin{tabular}{cc|ccc}
\hline
                                &           & \multicolumn{3}{c}{Logits matching loss}                          \\ \hline
\multicolumn{1}{c|}{Teacher}    & Student   & \multicolumn{1}{c|}{MSE}   & \multicolumn{1}{c|}{Cos Sim} & CE    \\ \hline
\multicolumn{1}{c|}{ResNet32x4} & ResNet8x4 & \multicolumn{1}{c|}{76.52} & \multicolumn{1}{c|}{76.86}   & 75.24 \\ \hline
\multicolumn{1}{c|}{WRN-40-2}   & WRN-40-1  & \multicolumn{1}{c|}{75.14} & \multicolumn{1}{c|}{75.51}   & 74.27 \\ \hline
\end{tabular}
\end{table}

\begin{figure}[!htbp]
    \centering
    \subfigure[training]{
        \includegraphics[width=.30\textwidth]{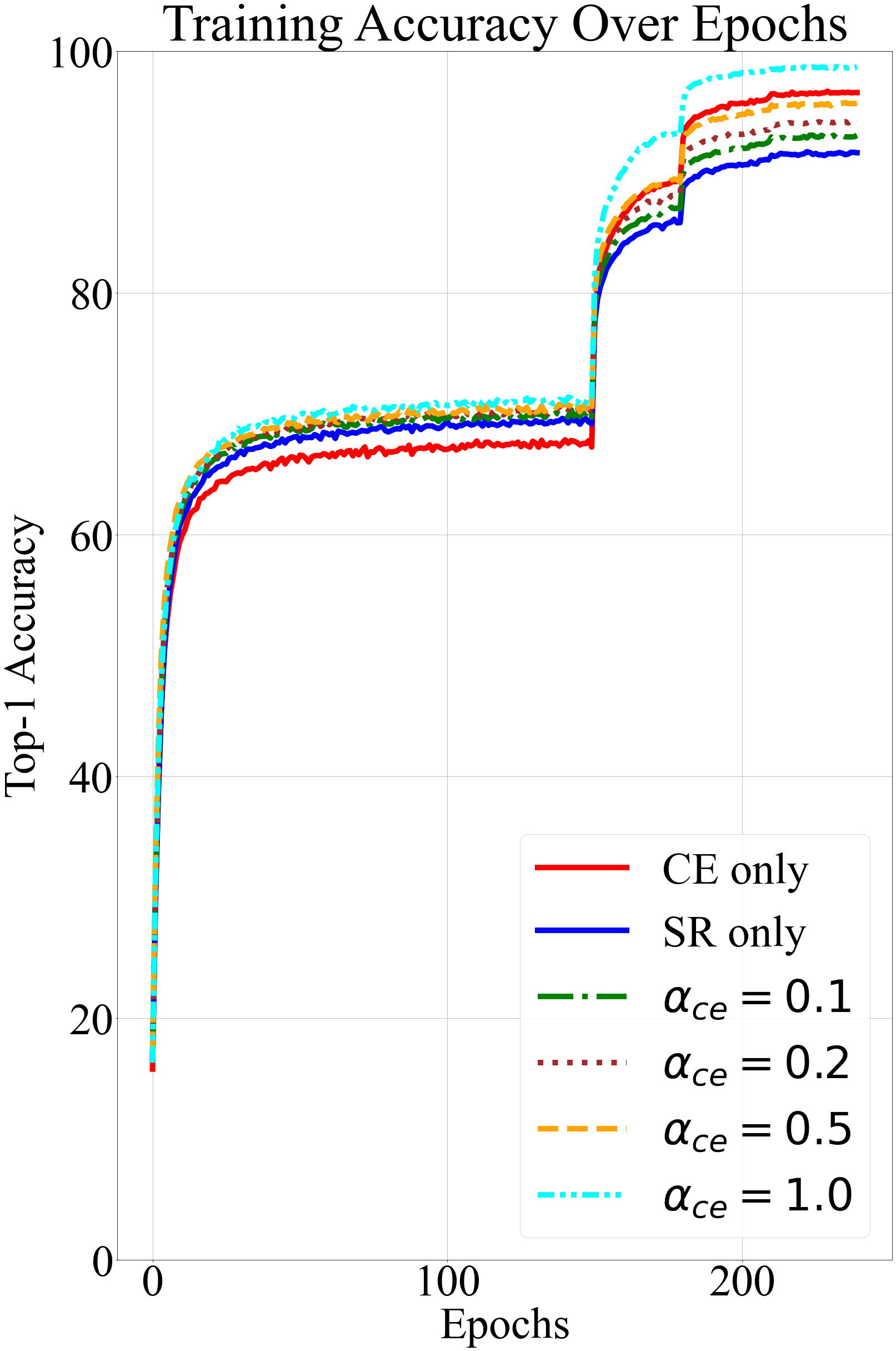}
    }
    \subfigure[validation]{
        \includegraphics[width=.30\textwidth]{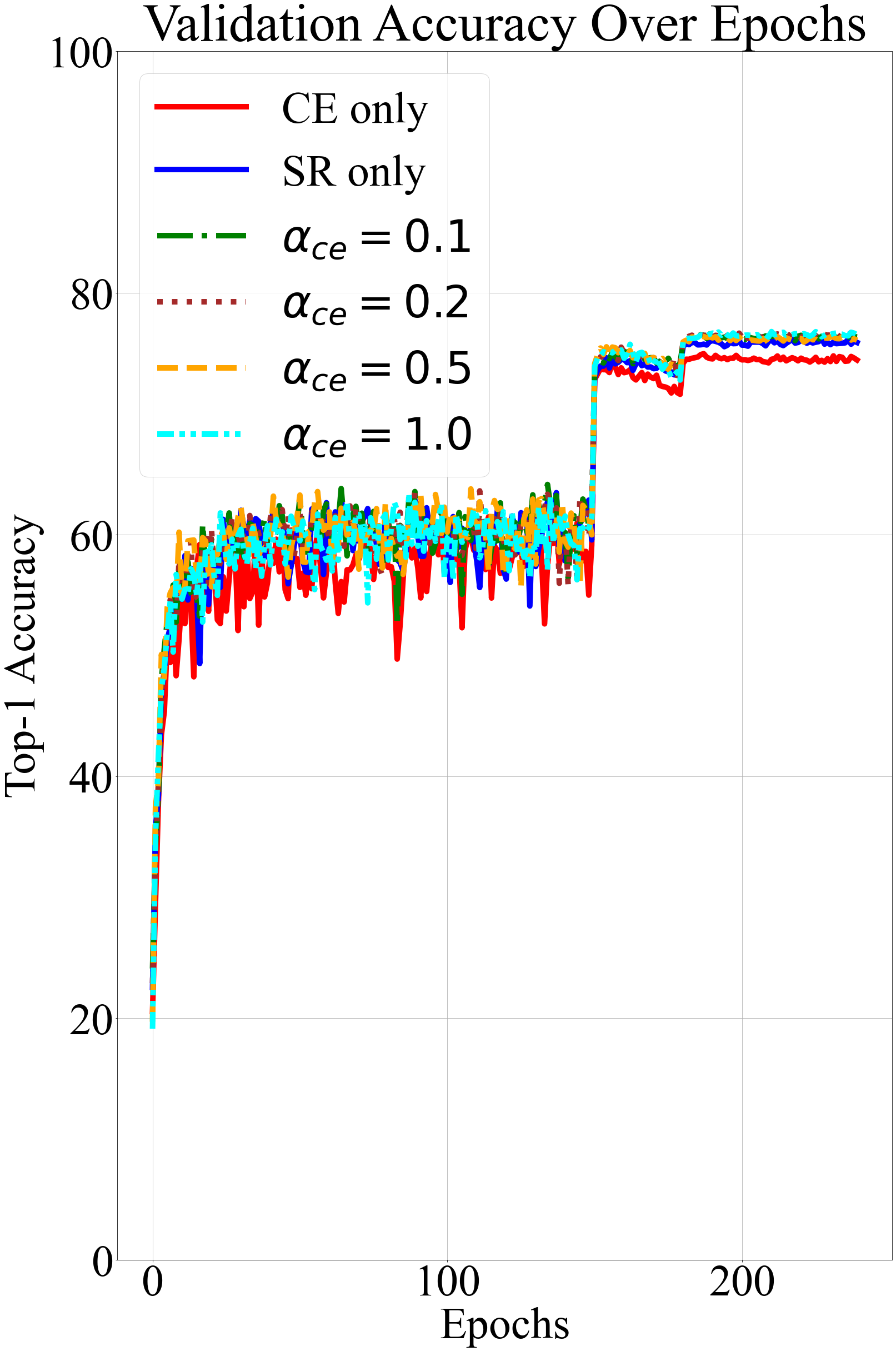}
    }
    \subfigure[\textcolor{black}{Zoom-in view of validation after 150 epochs.}]{
        \includegraphics[width=.30\textwidth]{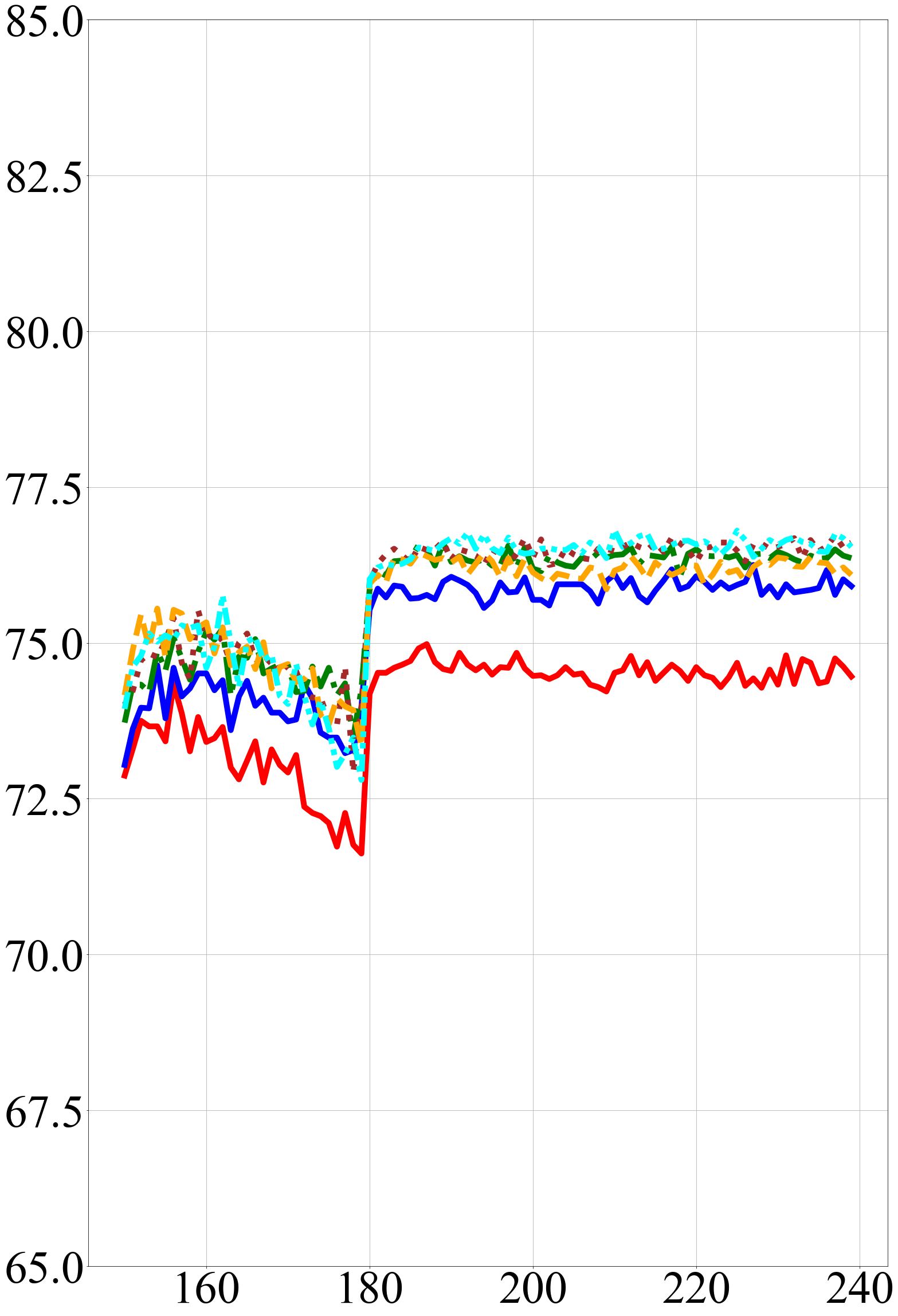}
    }
    \caption{Top-1 accuracy curves for different $\alpha_{ce}$.}
    \label{fig:alpha}
\end{figure}

\begin{figure}[!htbp]
    \centering
    \subfigure[training]{
        \includegraphics[width=.30\textwidth]{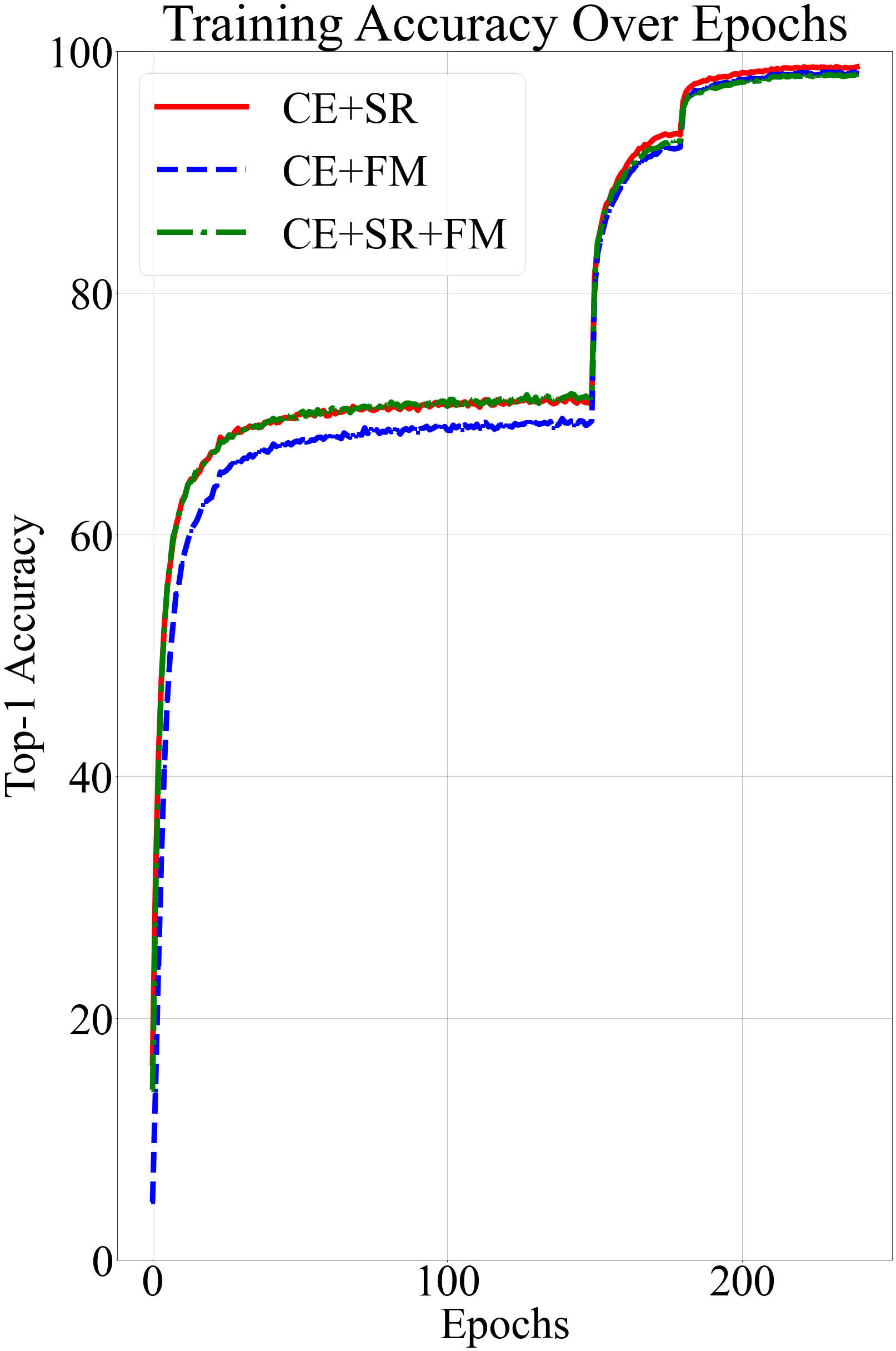}
    }
    \subfigure[validation]{
        \includegraphics[width=.30\textwidth]{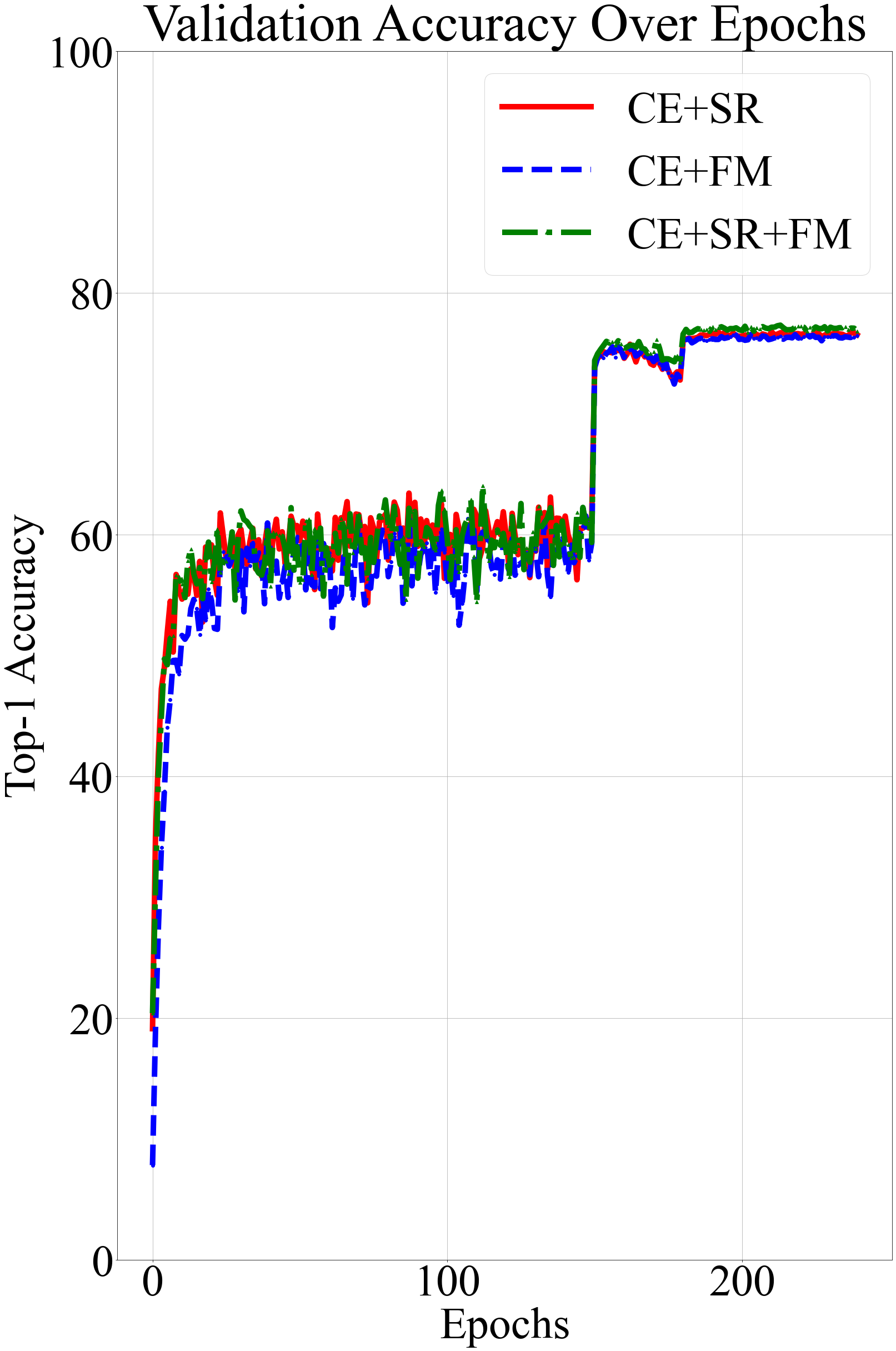}
    }
    \subfigure[\textcolor{black}{Zoom-in view of validation after 150 epochs.}]{
        \includegraphics[width=.30\textwidth]{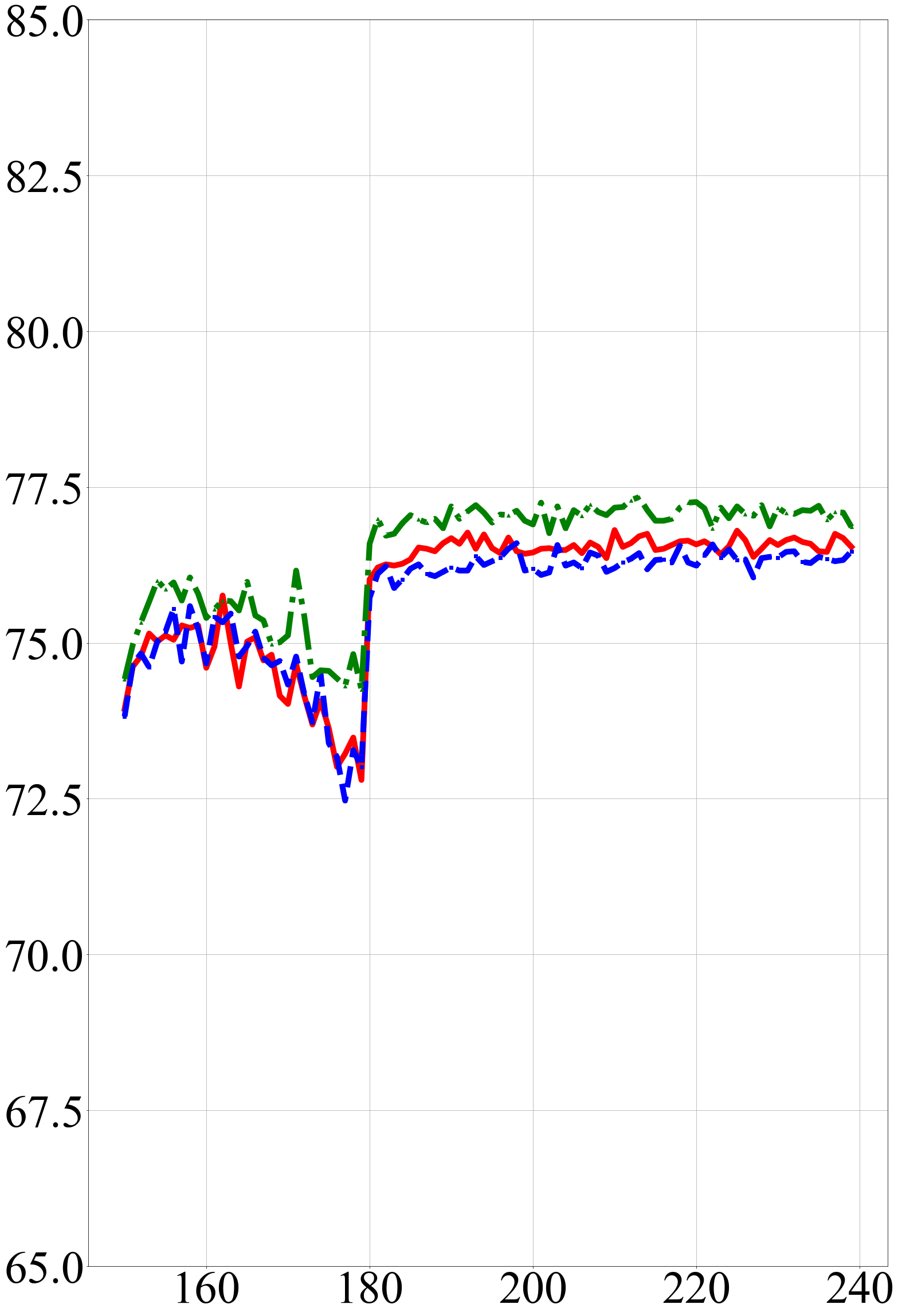}
    }
    \caption{Top-1 accuracy curves for different loss combination.}
    \label{fig:comb}
\end{figure}

\subsection{Results on ImageNet}
\textcolor{black}{In this section, we present the results on the ImageNet \cite{deng2009imagenet} dataset. We used the SGD optimizer with the same momentum parameter as we did for the CIFAR-100 dataset, and applied a weight decay factor of $1 \times 10^{-4}$. For the training process, we utilized a batch size of 512 and commenced with a learning rate of 0.2. Our learning rate schedule involved reducing the learning rate by a factor of 10 at predefined epochs: 30th, 60th, and 90th, completing the training at the 120th epoch. we employed a linear combination of the cross-entropy loss ($l^{ce}$) and the softmax regression loss ($l^{sr}$) as we did for CIFAR-100 training.}

\textcolor{black}{The chosen architecture for the connector layers was a 1×1Conv-3×3Conv-1×1Conv configuration. The selection of distinct settings for the CIFAR-100 and ImageNet experiments was informed by insights from the SimKD settings \cite{chen2022knowledge}. A single-layer transformation may be insufficient for precise alignment due to the significant capability disparity between the teacher and student models. Considering the intricacy and size of the ImageNet dataset, the 1×1Conv-3×3Conv-1×1Conv connector was selected for its ability to achieve accurate alignment between the teacher and student models, thereby enhancing performance. Additionally, as indicated in \cite{chen2022knowledge}, this configuration does not impose considerable computational overhead.}

\textcolor{black}{The experimental results for IJCKD, SRRL, and other comparative methods are detailed in Table \ref{tabel:imagenet}, which illustrates performance metrics such as top-1 and top-5 accuracy. Notably, IJCKD consistently outperformed SRRL, demonstrating its superior efficacy in enhancing model performance on large-scale datasets, including ImageNet. For the ResNet-18 model, IJCKD achieved significant improvements, with increases of 0.51\% in top-1 accuracy and 0.46\% in top-5 accuracy. In the context of the MobileNet model, the improvements were even more pronounced, with IJCKD leading to a 2.18\% enhancement in top-1 accuracy and a 1.46\% rise in top-5 accuracy. These results emphatically affirm the effectiveness of the IJCKD approach, especially in the realm of complex, high-dimensional datasets.}

\begin{table}[!htbp]
\caption{Top-1 and top-5 accuracy (\%) on ImageNet.}
\label{tabel:imagenet}
\centering
\begin{tabular}{c|cc|cc}
\hline
        & \multicolumn{2}{c|}{\begin{tabular}[c]{@{}c@{}}T: ResNet-34\\ S: ResNet-18\end{tabular}} & \multicolumn{2}{c}{\begin{tabular}[c]{@{}c@{}}T: ResNet-50\\ S: MobileNet-v1\end{tabular}} \\
        & top-1                        & top-5                        & top-1                          & top-5                         \\\hline
Teacher & 73.31                        & 91.42                        & 76.16                          & 92.87                         \\
Student & 69.75                        & 89.07                        & 68.87                          & 88.76                         \\\hline
KD      & 70.67                        & 90.04                        & 70.49                          & 89.92                         \\
AT      & 71.03                        & 90.04                        & 70.18                          & 89.68                         \\
CRD     & 71.17                        & 90.13                        & 69.07                          & 88.94                         \\
WSLL    & 72.04                        & 90.70                         & 71.52                          & 90.34                         \\
DKD &71.70 & 90.41 &  72.05
 & 91.05\\
\hline
SRRL    & 71.73                        & 90.60                         & 72.49                          & 90.92                         \\
IJCKD   & \textbf{72.24}                             & \textbf{91.06}                             & \textbf{74.67}                               &\textbf{92.38}  \\
\hline
\end{tabular}
\end{table}

\subsection{Validation of the Ideal Joint Classifier Assumption}

\subsubsection{Can teacher classifier be the ideal joint classifier?} 
To evaluate the ideal joint classifier assumption that states the joint classifier should achieve lower risks on both teacher and student representations, we compared the performance of a student trained with the teacher's classifier to a student trained with its own classifier. Specifically, we evaluated four teacher-student pairs on CIFAR-100 and reported the top-1 accuracy in Table \ref{tabel:classifier}. The student networks were only trained with hard labels but using the pre-trained teacher classifier. The results show that training the student with the teacher and its own classifier achieved close results across all four teacher-student pairs.

This supports our assumption of the existence of the ideal joint classifier. However, we also observed that in some cases, the student trained with a teacher classifier obtained significantly higher accuracy. For example, the ResNet8x4 trained with ResNet32x4's classifier achieved 1.96\% higher accuracy. This suggests that other factors may be at play and there are still unexplored mechanisms behind resuing the pre-trained teacher classifier. Overall, our experiments on CIFAR-100 provide evidence that the ideal joint classifier is a valuable concept in knowledge distillation. Further research is needed to explore its properties and potential applications.

\begin{table}[!htbp]
\caption{Top-1 accuracy for the student backbone with teacher and its own classifier.}
\label{tabel:classifier}
\centering
\begin{tabular}{cc|cc}
\hline
Teacher    & Student   & \begin{tabular}[c]{@{}c@{}}Teacher's\\ classifer\end{tabular} & \begin{tabular}[c]{@{}c@{}}Student's\\ classifer\end{tabular} \\\hline
WRN-40-2   & WRN-40-1  & 72.32               & 71.98               \\
ResNet56   & ResNet20  & 68.91               & 69.06               \\
ResNet110  & ResNet32  & 70.98               & 71.14               \\
ResNet32x4 & ResNet8x4 & 74.46               & 72.50              \\\hline 
\end{tabular}
\end{table}

\subsubsection{Evolution of Classifier Alignment During SRRL Training}
In order to provide additional compelling evidence in support of the Ideal Joint Classifier Assumption, we conducted a series of experimental investigations aimed at observing the evolution of the relationship between the teacher's classifier and the student's classifier, particularly when the student network was trained under the SRRL framework. The objective was to elucidate whether the alignment between these classifiers dynamically evolves during the course of training, thereby further substantiating the Ideal Joint Classifier Assumption.

Figure \ref{fig:dist}, illustrates the Frobenius Norm of the difference in weights between the teacher's classifier and the student's classifier over the course of SRRL training. The observed trends unveil a noteworthy pattern: as the performance of the student network progressively improves, there is a discernible trend towards a reduction in the distance between the teacher's classifier and the student's classifier. This empirical observation suggests that, indeed, as the student becomes more proficient, there is a growing convergence towards a shared classifier, corroborating the Ideal Joint Classifier Assumption.

These findings, in conjunction with our previous comparative analysis of student performance when trained with teacher classifiers versus their own classifiers, collectively underscore the validity of the Ideal Joint Classifier Assumption. It is apparent that the alignment of classifiers between the teacher and student networks benefits the performance of knowledge distillation algorithms, further reinforcing the central concept underlying our framework.

\begin{figure}[!htbp]
    \centering
    \subfigure[Teacher: ResNet32x4; Student: ResNet8x4]{
        \includegraphics[width=.3\textwidth]{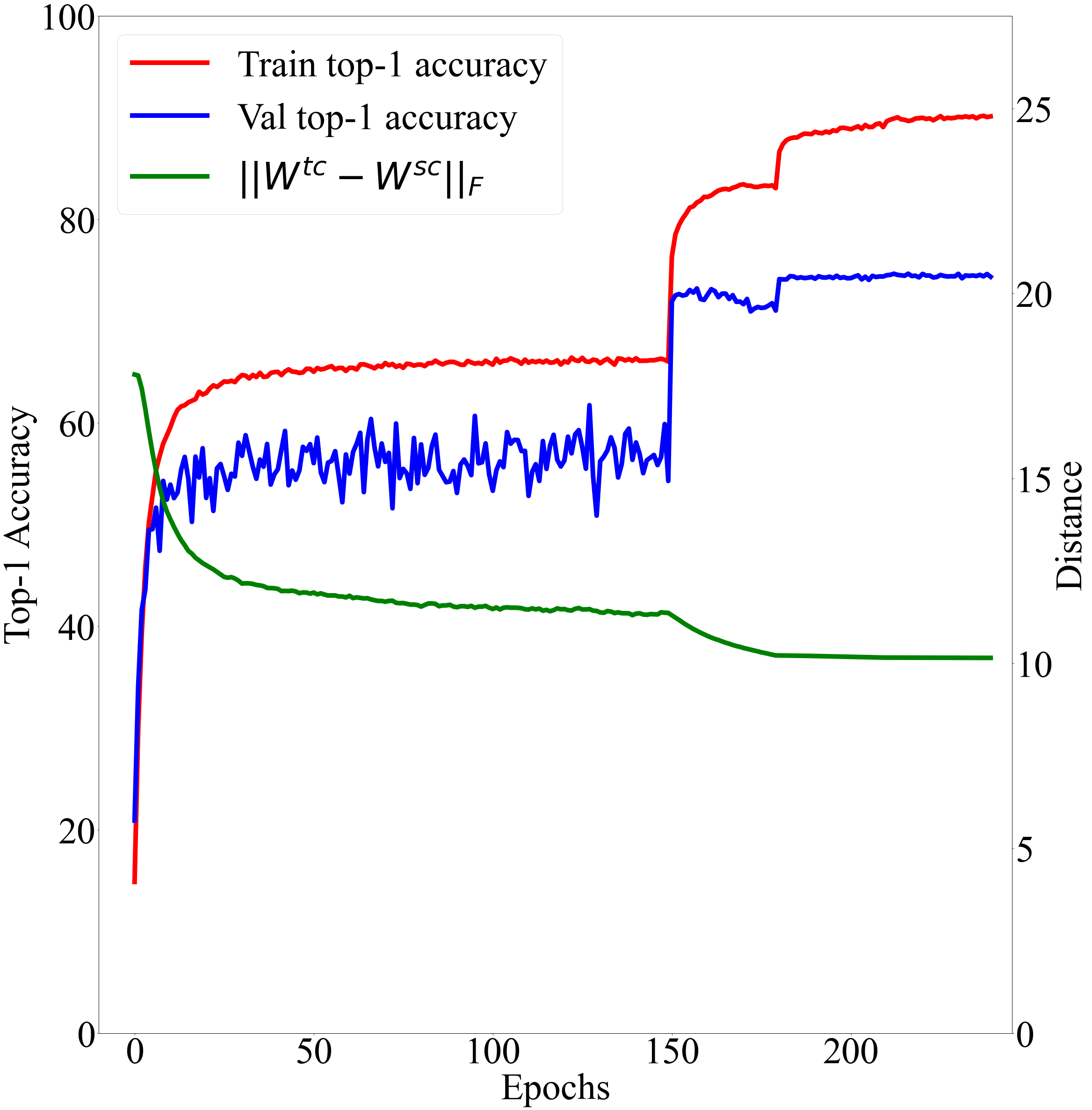}
    }
    \subfigure[Teacher: ResNet110; Student: ResNet32]{
        \includegraphics[width=.3\textwidth]{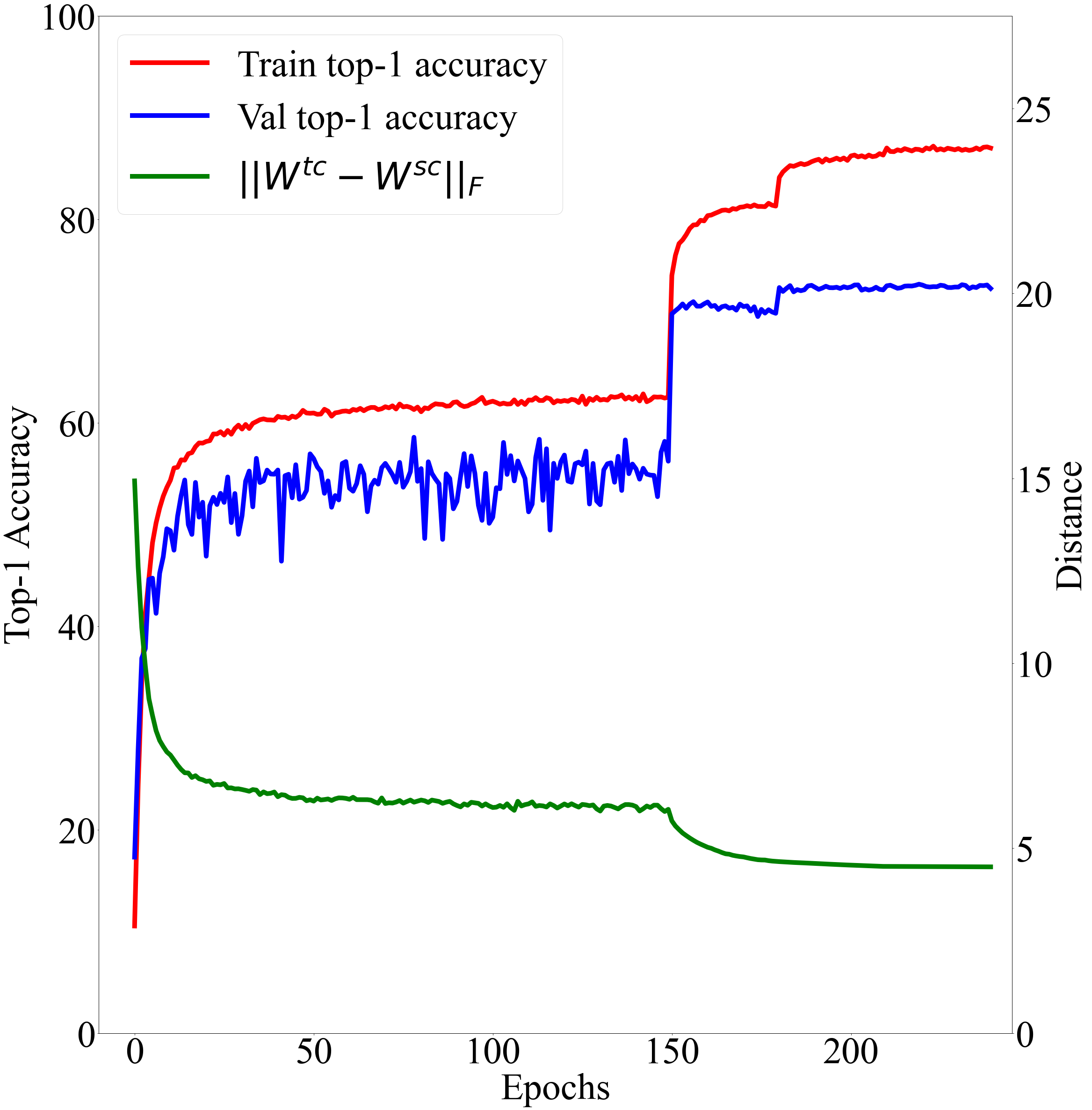}
    }
    \subfigure[Teacher: ResNet110; Student: ResNet20]{
        \includegraphics[width=.3\textwidth]{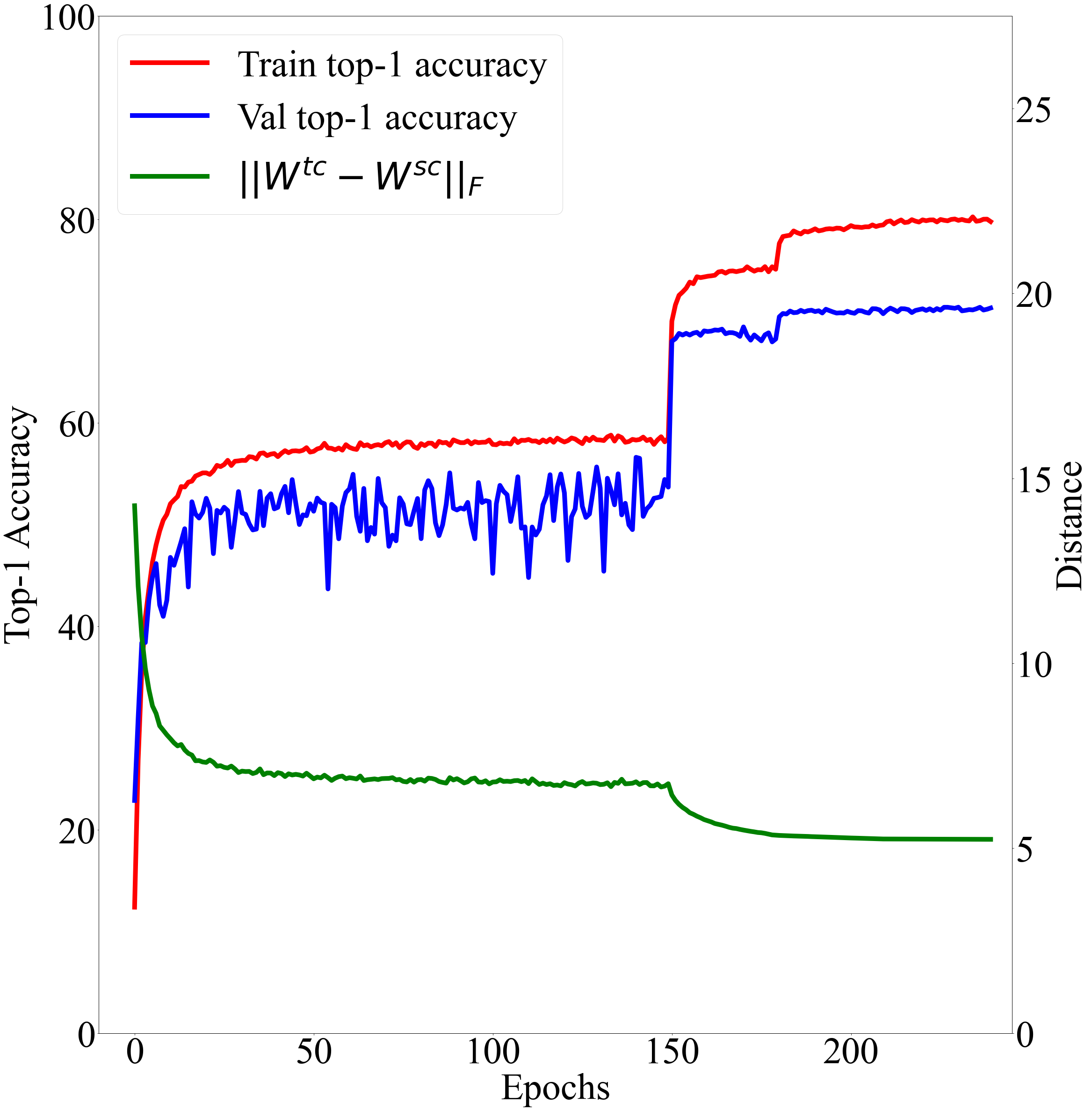}
    }
    \caption{Frobenius Norm of the difference of the weight of teacher's classifier ($W^{tc}$) and student's classifier ($W^{sc}$)  during SRRL training.}
    \label{fig:dist}
\end{figure}

\subsubsection{\textcolor{black}{Alternative Implementations of IJCKD}}

\textcolor{black}{To further assess the versatility of the IJCKD framework, two alternative implementations are presented, extending beyond the mere reuse of the teacher's classifier. Each variant possesses unique characteristics and specific implementation details. The first alternative involves the online training of a joint classifier, as detailed in Algorithm \ref{alg:joint}. This method emphasizes the use of a jointly trained classifier to align teacher and student network representations. The procedure encompasses several critical steps: extracting features from both the student and teacher networks, generating logits using the joint classifier for both networks, and computing logits matching loss between the student's and teacher's logits. Additionally, the cross-entropy losses for both sets of logits are calculated and scaled by the factor $\alpha$. In this implementation, $\alpha$ is set to 0.2, and $\beta$ for scaling the logits matching loss is set to 1.0. This value was chosen based on the premise that the teacher's representations are already well-trained, and a higher scale of the teacher's cross-entropy loss would better balance the overall scale.}

\textcolor{black}{The second variant, as illustrated in Algorithm \ref{alg:penalty}, incorporates a distance penalty to bridge the gap between the teacher and student classifiers in the distillation process. This implementation consists of several key components: extracting logits and features directly from both the student and teacher networks, computing logits matching loss between the two, and calculating the cross-entropy loss based on the student's logits and the provided hard label. A distinctive element of this approach is the introduction of a distance penalty, determined by the norm of the discrepancy in weights of the final fully connected layers of both networks. We simply set both the two scaling factors $\alpha$ and $\beta$ to 1. This method focuses on reducing the differences between the teacher and student classifiers, guiding the student classifier towards an approximation of the ideal joint classifier.}

\textcolor{black}{Table \ref{tabel:alternatives} presents a comparative analysis of three different implementations of IJCKD across two network pairs, ResNet32x4-ResNet8x4 and WRN-40-2-WRN-40-1, on the CIFAR-100 dataset. For the ResNet32x4-ResNet8x4 pairing, the method utilizing the teacher's classifier reached a top-1 accuracy of 76.52\%, with the joint training approach marginally surpassing it at 76.65\%. The distance penalty variant was also effective, achieving a top-1 accuracy of 76.55\%. In the case of the WRN-40-2-WRN-40-1 pair, the teacher's classifier strategy led with a top-1 accuracy of 75.14\%, followed by the joint training method at 75.02\% and the distance penalty method at 74.69\%. These results indicate that IJCKD, as a comprehensive framework, is adaptable to various implementations, each showing promising results.}

\begin{algorithm}
\begin{lstlisting}[style=pytorchstyle]
#input: x, hard label: y
#student network: net_s, teacher network: net_t
#connector for align teacher and student channels
#joint_classifier get by online joint training
_, feat_s = net_s(x)
feat_s = connector(feat_s)
_, feat_t = net_t(x)

logits_t = joint_classifier(avg_pool(feat_t))
logits_s = joint_classifier(avg_pool(feat_s))

loss_lm = F.mse_loss(logits_s, logits_t)

#alpha is used to scale the teacher and student cross entropy losses

loss_ce = alpha*F.cross_entropy(logits_s,y) + (1-alpha)*F.cross_entropy(logits_t,y)

#beta is used to scale the losses
loss_ijckd = loss_ce + beta*loss_lm

\end{lstlisting}
\caption{PyTorch Code for IJCKD with joint training.}
\label{alg:joint}
\end{algorithm}

\begin{algorithm}
\begin{lstlisting}[style=pytorchstyle]
#input: x, hard label: y
#student network: net_s, teacher network: net_t

logits_s, feat_s = net_s(x)
logits_t, feat_t = net_t(x)

loss_lm = F.mse_loss(logits_s, logits_t)
loss_ce = F.cross_entropy(logits_s,y)
dist_penalty = torch.norm(net_t.fc.weight-net_s.fc.weight) # Default: Frobenius norm

#alpha and beta are used to scale the losses

loss_ijckd = loss_ce + alpha*loss_lm + beta*dist_penalty

\end{lstlisting}
\caption{PyTorch Code for IJCKD with distance penalty.}
\label{alg:penalty}
\end{algorithm}

\begin{table}[]
\caption{Comparison of three different implementations of IJCKD.}
\label{tabel:alternatives}
\centering
\begin{tabular}{c|c|c|c|c}
\hline
Teacher    & Student   & Teacher's Classifier & Joint Training & Distance Penalty \\ \hline
ResNet32x4 & ResNet8x4 & 76.52                       & 76.65                 & 76.55        \\ \hline
WRN-40-2   & WRN-40-1  & 75.14                       & 75.02                 & 74.69   \\ \hline
\end{tabular}
\end{table}
\section{Discussion}

\subsection{Scope}  
\textcolor{black}{This paper presents the development and examination of the proposed IJCKD framework, which is built on the observation of previous works, SRRL and SimKD. Its main contribution is the ideal joint classifier assumption, which refines error bounds into optimization objectives and underscores the effectiveness of using the teacher's classifier in knowledge distillation. The IJCKD's adaptability is illustrated through three different implementations based on the ideal joint classifier assumption, showcasing its potential as a broad framework for new distillation algorithms. The IJCKD framework offers a novel perspective in knowledge distillation research but also encourages a more systematic, theory-based approach to developing distillation methods. Acknowledging the foundational nature of this assumption, the paper highlights the IJCKD framework's potential for further exploration and diverse applications, emphasizing its strength in integrating new ideas and deepening the understanding of knowledge distillation.}

\subsection{Understanding the 'Ideal Joint Classifier Assumption'} 
The 'Ideal Joint Classifier Assumption' refers to a classifier that achieves the lowest risk for the student and teacher representations, as described in Eq. \ref{eq:ideal joint classifier} of our paper. In our work, we first adopt the setting of SimKD, where the teacher's classifier is reused as the ideal joint classifier. Note that the classifier in this context refers to an output layer of the neural network (e.g., the layer before the softmax). This choice using the SimKD setting is rooted in the teacher's classifier has already minimized the risk on their own representations; however, we extend this concept by recognizing that the ideal joint classifier should also jointly minimize the risk on the student representations. To achieve this, we include the cross-entropy loss (the first term in Eq. \ref{eq:IJCKD}) to minimize the risk of the teacher's classifier over the student representation with respect to the corresponding hard labels. Since the teacher’s classifier is fixed in this context, we adapt the student representation with respect to the teacher's classifier to ensure that the teacher's classifier achieve the lowest risk on the student representations.

\subsection{IJCKD as a Corrective Step} 
The IJCKD framework can be viewed as a corrective step in comparison to SimKD and SRRL. Where IJCKD distinguishes itself is by acknowledging that the ideal joint classifier should not solely minimize the risk of the teacher's classifier over the teacher representation but should also jointly minimize the risk of the student representation. To achieve this, we introduce the cross-entropy loss (the first term in Eq. \ref{eq:IJCKD}) to minimize the risk of the teacher's classifier concerning the student representation using corresponding hard labels. In this context, the teacher's classifier remains fixed, and the student representation adapts to ensure that the teacher's classifier achieves the lowest risk concerning the student representations. It enforces the reused teacher classifier to become the ideal joint classifier by introducing the cross-entropy loss term. This distinction becomes particularly significant when dealing with scenarios where the reused teacher is not inherently the ideal joint classifier, necessitating a more potent connector to effectively align the teacher and student representations.

\subsection{Limitations and Future Research} 
\textcolor{black}{Although the experimental results validate the IJCKD framework, there are concerns to consider. While IJCKD shares basic principles with SRRL and SimKD, its distinct approach, centered on the ideal joint classifier assumption, leads to unique challenges. A key success factor for these methods is the shared classifier between teacher and student networks. However, merely reusing the teacher's classifier might restrict IJCKD's applicability in certain contexts. The proposed alternative implementations partially mitigate this, yet the intricate balance between classifier discrepancy and adaptability remains less explored. It is essential to delve deeper into this aspect and examine the nuances of classifier sharing and its impact on the versatility and effectiveness of the IJCKD framework. Thus, further investigations could focus on optimizing the balance between maintaining model accuracy and ensuring sufficient flexibility for diverse applications, potentially broadening the scope and utility of IJCKD in various knowledge distillation scenarios.}

\textcolor{black}{Also, there is a perception that the framework might be overly idealized, with concerns about the potential looseness of the error bound it proposes. However, it's important to emphasize that the significance of this error bound transcends its immediate numerical implications. The value of this error bound lies in its role as a fundamental concept, aiding in the comprehension of the complex interactions between teacher and student networks in the knowledge distillation process. It provides a theoretical lens through which the effectiveness and efficiency of knowledge transfer can be analyzed, offering insights into how different network architectures and training strategies might influence the distillation outcome. This understanding is crucial for refining the distillation techniques and advancing the field, pushing beyond mere empirical results to a more nuanced, theory-driven approach. However, future work should aim to refine and tighten this error bound, particularly in scenarios where there is a large discrepancy between the teacher and student representations. This would enhance the precision and applicability of the IJCKD framework, enabling more effective knowledge transfer in diverse and challenging distillation contexts.}

\textcolor{black}{The potential of the IJCKD framework extends beyond its current applications, heralding a new era of research that could revolutionize knowledge distillation across diverse sectors, including industrial inspection and beyond, as referenced in \cite{c2,c3,c4}. Its inherent flexibility showcases its applicability across a myriad of learning environments, fostering a fertile ground for innovative research. This adaptability not only opens new vistas for the development of more advanced and effective distillation methodologies but also invites the exploration of IJCKD's applicability in more specific and targeted areas. Future research should focus on tailoring the IJCKD framework to specific domains, thereby unlocking its full potential in specialized applications. This could involve adapting the framework to unique challenges and data characteristics of different fields, ranging from healthcare and finance to autonomous systems and beyond. Such targeted exploration could lead to highly specialized distillation techniques that are more efficient and effective in their respective areas. The adaptability and versatility of the IJCKD framework make it an ideal candidate for this kind of focused research, promising to yield significant advancements in both the theory and practice of knowledge distillation.}

\section{Conclusion}

In summary, this paper theoretically analyzed the softmax regression-based representation learning, including two representative methods, SRRL and SimKD. We established an error bound which upper bounded the student's error by the teacher's error with the disagreement term between student and teacher output logits under the proposed ideal joint classifier assumption. Further, IJCKD, a novel knowledge distillation framework was proposed that unifies the previous work based on the idea of softmax regression. Our experiments demonstrate the effectiveness of IJCKD, which consistently outperformed state-of-the-art methods on a variety of benchmarks. Our results suggest that IJCKD is a versatile and effective method for knowledge distillation, which can be adapted to various architectures and datasets. Overall, we believe that our work provides a valuable contribution to the field of knowledge distillation and can be applied to a wide range of practical applications.

\section*{Acknowledgement}
This work was supported by grants from the National Heart, Lung, and Blood Institute (\#R21HL159661), and the National Science Foundation (IUCRC \#2052528 and CAREER \#1943552).

\bibliographystyle{elsarticle-num} 
\bibliography{refs}

\begin{thebibliography}{10}
\expandafter\ifx\csname url\endcsname\relax
  \def\url#1{\texttt{#1}}\fi
\expandafter\ifx\csname urlprefix\endcsname\relax\def\urlprefix{URL }\fi
\expandafter\ifx\csname href\endcsname\relax
  \def\href#1#2{#2} \def\path#1{#1}\fi

\bibitem{krizhevsky2017imagenet}
A.~Krizhevsky, I.~Sutskever, G.~E. Hinton, Imagenet classification with deep convolutional neural networks, Communications of the ACM 60~(6) (2017) 84--90.

\bibitem{he2016deep}
K.~He, X.~Zhang, S.~Ren, J.~Sun, Deep residual learning for image recognition, in: Proceedings of the IEEE conference on computer vision and pattern recognition, 2016, pp. 770--778.

\bibitem{long2015fully}
J.~Long, E.~Shelhamer, T.~Darrell, Fully convolutional networks for semantic segmentation, in: Proceedings of the IEEE conference on computer vision and pattern recognition, 2015, pp. 3431--3440.

\bibitem{chen2017deeplab}
L.-C. Chen, G.~Papandreou, I.~Kokkinos, K.~Murphy, A.~L. Yuille, Deeplab: Semantic image segmentation with deep convolutional nets, atrous convolution, and fully connected crfs, IEEE transactions on pattern analysis and machine intelligence 40~(4) (2017) 834--848.

\bibitem{zhao2017pyramid}
H.~Zhao, J.~Shi, X.~Qi, X.~Wang, J.~Jia, Pyramid scene parsing network, in: Proceedings of the IEEE conference on computer vision and pattern recognition, 2017, pp. 2881--2890.

\bibitem{girshick2014rich}
R.~Girshick, J.~Donahue, T.~Darrell, J.~Malik, Rich feature hierarchies for accurate object detection and semantic segmentation, in: Proceedings of the IEEE conference on computer vision and pattern recognition, 2014, pp. 580--587.

\bibitem{redmon2016you}
J.~Redmon, S.~Divvala, R.~Girshick, A.~Farhadi, You only look once: Unified, real-time object detection, in: Proceedings of the IEEE conference on computer vision and pattern recognition, 2016, pp. 779--788.

\bibitem{howard2017mobilenets}
A.~G. Howard, M.~Zhu, B.~Chen, D.~Kalenichenko, W.~Wang, T.~Weyand, M.~Andreetto, H.~Adam, Mobilenets: Efficient convolutional neural networks for mobile vision applications, arXiv preprint arXiv:1704.04861 (2017).

\bibitem{zhang2018shufflenet}
X.~Zhang, X.~Zhou, M.~Lin, J.~Sun, Shufflenet: An extremely efficient convolutional neural network for mobile devices, in: Proceedings of the IEEE conference on computer vision and pattern recognition, 2018, pp. 6848--6856.

\bibitem{han2015deep}
S.~Han, H.~Mao, W.~J. Dally, Deep compression: Compressing deep neural networks with pruning, trained quantization and huffman coding, arXiv preprint arXiv:1510.00149 (2015).

\bibitem{tai2015convolutional}
C.~Tai, T.~Xiao, Y.~Zhang, X.~Wang, et~al., Convolutional neural networks with low-rank regularization, arXiv preprint arXiv:1511.06067 (2015).

\bibitem{bucilu2006model}
C.~Buciluǎ, R.~Caruana, A.~Niculescu-Mizil, Model compression, in: Proceedings of the 12th ACM SIGKDD international conference on Knowledge discovery and data mining, 2006, pp. 535--541.

\bibitem{hinton2015distilling}
G.~Hinton, O.~Vinyals, J.~Dean, Distilling the knowledge in a neural network, arXiv preprint arXiv:1503.02531 (2015).

\bibitem{romero2014fitnets}
A.~Romero, N.~Ballas, S.~E. Kahou, A.~Chassang, C.~Gatta, Y.~Bengio, Fitnets: Hints for thin deep nets, arXiv preprint arXiv:1412.6550 (2014).

\bibitem{zagoruyko2016paying}
S.~Zagoruyko, N.~Komodakis, Paying more attention to attention: Improving the performance of convolutional neural networks via attention transfer, arXiv preprint arXiv:1612.03928 (2016).

\bibitem{chen2021cross}
D.~Chen, J.-P. Mei, Y.~Zhang, C.~Wang, Z.~Wang, Y.~Feng, C.~Chen, Cross-layer distillation with semantic calibration, in: Proceedings of the AAAI Conference on Artificial Intelligence, Vol.~35, 2021, pp. 7028--7036.

\bibitem{chen2021distilling}
P.~Chen, S.~Liu, H.~Zhao, J.~Jia, Distilling knowledge via knowledge review, in: Proceedings of the IEEE/CVF Conference on Computer Vision and Pattern Recognition, 2021, pp. 5008--5017.

\bibitem{tung2019similarity}
F.~Tung, G.~Mori, Similarity-preserving knowledge distillation, in: Proceedings of the IEEE/CVF International Conference on Computer Vision, 2019, pp. 1365--1374.

\bibitem{tian2019contrastive}
Y.~Tian, D.~Krishnan, P.~Isola, Contrastive representation distillation, arXiv preprint arXiv:1910.10699 (2019).

\bibitem{zhou2021rethinking}
H.~Zhou, L.~Song, J.~Chen, Y.~Zhou, G.~Wang, J.~Yuan, Q.~Zhang, Rethinking soft labels for knowledge distillation: A bias-variance tradeoff perspective, arXiv preprint arXiv:2102.00650 (2021).

\bibitem{zhao2022decoupled}
B.~Zhao, Q.~Cui, R.~Song, Y.~Qiu, J.~Liang, Decoupled knowledge distillation, in: Proceedings of the IEEE/CVF Conference on computer vision and pattern recognition, 2022, pp. 11953--11962.

\bibitem{yang2021knowledge}
J.~Yang, B.~Martinez, A.~Bulat, G.~Tzimiropoulos, et~al., Knowledge distillation via softmax regression representation learning, International Conference on Learning Representations (ICLR), 2021.

\bibitem{chen2022knowledge}
D.~Chen, J.-P. Mei, H.~Zhang, C.~Wang, Y.~Feng, C.~Chen, Knowledge distillation with the reused teacher classifier, in: Proceedings of the IEEE/CVF conference on computer vision and pattern recognition, 2022, pp. 11933--11942.

\bibitem{ben2010theory}
S.~Ben-David, J.~Blitzer, K.~Crammer, A.~Kulesza, F.~Pereira, J.~W. Vaughan, A theory of learning from different domains, Machine learning 79~(1) (2010) 151--175.

\bibitem{hornik1989multilayer}
K.~Hornik, M.~Stinchcombe, H.~White, Multilayer feedforward networks are universal approximators, Neural networks 2~(5) (1989) 359--366.

\bibitem{scarselli1998universal}
F.~Scarselli, A.~C. Tsoi, Universal approximation using feedforward neural networks: A survey of some existing methods, and some new results, Neural networks 11~(1) (1998) 15--37.

\bibitem{kim2021comparing}
T.~Kim, J.~Oh, N.~Kim, S.~Cho, S.-Y. Yun, Comparing kullback-leibler divergence and mean squared error loss in knowledge distillation, arXiv preprint arXiv:2105.08919 (2021).

\bibitem{gou2023multi}
J.~Gou, X.~Xiong, B.~Yu, L.~Du, Y.~Zhan, D.~Tao, Multi-target knowledge distillation via student self-reflection, International Journal of Computer Vision (2023) 1--18.

\bibitem{heo2019comprehensive}
B.~Heo, J.~Kim, S.~Yun, H.~Park, N.~Kwak, J.~Y. Choi, A comprehensive overhaul of feature distillation, in: Proceedings of the IEEE/CVF International Conference on Computer Vision, 2019, pp. 1921--1930.

\bibitem{park2019relational}
W.~Park, D.~Kim, Y.~Lu, M.~Cho, Relational knowledge distillation, in: Proceedings of the IEEE/CVF Conference on Computer Vision and Pattern Recognition, 2019, pp. 3967--3976.

\bibitem{krizhevsky2009learning}
A.~Krizhevsky, G.~Hinton, et~al., Learning multiple layers of features from tiny images (2009).

\bibitem{deng2009imagenet}
J.~Deng, W.~Dong, R.~Socher, L.-J. Li, K.~Li, L.~Fei-Fei, Imagenet: A large-scale hierarchical image database, in: 2009 IEEE conference on computer vision and pattern recognition, Ieee, 2009, pp. 248--255.

\bibitem{paszke2019pytorch}
A.~Paszke, S.~Gross, F.~Massa, A.~Lerer, J.~Bradbury, G.~Chanan, T.~Killeen, Z.~Lin, N.~Gimelshein, L.~Antiga, et~al., Pytorch: An imperative style, high-performance deep learning library, Advances in neural information processing systems 32 (2019).

\bibitem{ma2018shufflenet}
N.~Ma, X.~Zhang, H.-T. Zheng, J.~Sun, Shufflenet v2: Practical guidelines for efficient cnn architecture design, in: Proceedings of the European conference on computer vision (ECCV), 2018, pp. 116--131.

\bibitem{ioffe2015batch}
S.~Ioffe, C.~Szegedy, Batch normalization: Accelerating deep network training by reducing internal covariate shift, in: International conference on machine learning, pmlr, 2015, pp. 448--456.

\bibitem{nair2010rectified}
V.~Nair, G.~E. Hinton, Rectified linear units improve restricted boltzmann machines, in: Proceedings of the 27th international conference on machine learning (ICML-10), 2010, pp. 807--814.

\bibitem{ahn2019variational}
S.~Ahn, S.~X. Hu, A.~Damianou, N.~D. Lawrence, Z.~Dai, Variational information distillation for knowledge transfer, in: Proceedings of the IEEE/CVF Conference on Computer Vision and Pattern Recognition, 2019, pp. 9163--9171.

\bibitem{passalis2018learning}
N.~Passalis, A.~Tefas, Learning deep representations with probabilistic knowledge transfer, in: Proceedings of the European Conference on Computer Vision (ECCV), 2018, pp. 268--284.

\bibitem{c2}
C.~Hu, Y.~Wang, An efficient convolutional neural network model based on object-level attention mechanism for casting defect detection on radiography images, IEEE Transactions on Industrial Electronics 67~(12) (2020) 10922--10930.

\bibitem{c3}
C.~Zhao, C.~Hu, H.~Shao, Z.~Wang, Y.~Wang, Towards trustworthy multi-label sewer defect classification via evidential deep learning, in: ICASSP 2023-2023 IEEE International Conference on Acoustics, Speech and Signal Processing (ICASSP), IEEE, 2023, pp. 1--5.

\bibitem{c4}
C.~Hu, B.~Dong, H.~Shao, J.~Zhang, Y.~Wang, Toward purifying defect feature for multilabel sewer defect classification, IEEE Transactions on Instrumentation and Measurement 72 (2023) 1--11.

\end{thebibliography}

\end{document}